\documentclass[12pt,english]{article}
\usepackage{appendix}
\usepackage{mathptmx}

\usepackage[T1]{fontenc}
\usepackage[latin9]{inputenc}
\usepackage[a4paper]{geometry}
\usepackage{array}
\usepackage{verbatim}
\usepackage{rotating}
\usepackage{float}
\usepackage{bm}
\usepackage{multirow}
\usepackage{amsmath}
\usepackage{amsthm}
\usepackage{amssymb}
\usepackage{graphicx}
\usepackage{setspace}
\makeatletter
\usepackage{enumitem}

\newtheorem{condition}{Condition}[section]
\newtheorem{theorem}{Theorem}[section]

\newtheorem{proposition}{Proposition}[section]
\newtheorem{lemma}[theorem]{Lemma}

\providecommand{\tabularnewline}{\\}
\date{}
\usepackage{setspace}
\usepackage{algorithm,algpseudocode}
\algnewcommand\algorithmicinput{\textbf{Input:}}
\algnewcommand\INPUT{\item[\algorithmicinput]}
\algnewcommand\algorithmicpreproc{\textbf{Pre-processing}}
\algnewcommand\PREPROCESSING{\item[\algorithmicpreproc]}
\algnewcommand\algorithmicmonitoring{\textbf{Real-time monitoring}}
\algnewcommand\MONITORING{\item[\algorithmicmonitoring]}
\algnewcommand{\Initialize}[1]{%
  \State \textbf{Initialize:}
  \Statex \hspace*{\algorithmicindent}\parbox[t]{.8\linewidth}{\raggedright #1}
}
\algnewcommand\algorithmicoutput{\textbf{Output}}
\algnewcommand\OUTPUT{\item[\algorithmicoutput]}
\algnewcommand{\Output}[1]{%
  \State \textbf{Output:}
  \Statex \hspace*{\algorithmicindent}\parbox[t]{.8\linewidth}{\raggedright #1}
}

\usepackage{tikz}
\@ifundefined{showcaptionsetup}{}{%
 \PassOptionsToPackage{caption=false}{subfig}}
\usepackage{subfig}
\makeatother
\usepackage{babel}
\graphicspath{{./img/}}

\newif\ifworkinprogress
\workinprogresstrue
\ifworkinprogress
  \newcommand{\se}[1]{\textcolor{blue}{\textbf{[Samaneh] #1}}}
  \newcommand{\crn}[1]{\textcolor{magenta}{\textbf{[Chitta] #1}}}
  \newcommand{\kp}[1]{\textcolor{brown}{\textbf{[Kamran] #1}}}
\else
  \newcommand{\se}[1]{}
  \newcommand{\crn}[1]{}
  \newcommand{\kp}[1]{}
\fi

\usepackage[round,sort]{natbib}
\usepackage{authblk} % for authorship and all

\begin{document}
\selectlanguage{english}

\title{Large Multistream Data Analytics for Monitoring and Diagnostics in Manufacturing Systems}

\author {Samaneh Ebrahimi} 
\author {Chitta Ranjan} 
\author {Kamran Paynabar}

\affil {H. Milton Stewart School of Industrial and Systems Engineering,\\
          Georgia Institute of Technology, Atlanta, GA, USA}


\maketitle

\begin{abstract}
\begin{singlespace}
The high-dimensionality and volume of large scale multistream data has inhibited significant research progress in developing an \emph{integrated} monitoring and diagnostics (M\&D) approach. This data, also categorized as big data, is becoming common in manufacturing plants.
In this paper, we propose an integrated M\&D approach for large scale streaming data. We developed a novel monitoring method named Adaptive
Principal Component monitoring (APC) which adaptively chooses PCs
that are most likely to vary due to the change for early detection.
Importantly, we integrate a novel diagnostic approach, Principal Component Signal Recovery (PCSR), to enable a streamlined SPC. This diagnostics approach draws inspiration from
Compressed Sensing and uses Adaptive Lasso for identifying the sparse
change in the process. We theoretically motivate our approaches and
do a performance evaluation of our integrated M\&D method through simulations
and case studies.
\end{singlespace}
\end{abstract}

\section{Introduction \label{sec:Introduction}}

Recently, the problem of process monitoring and diagnosis using a large scale multi-stream data has become an active research area in statistical process control (SPC). The reason is two-fold: first, sensing technologies have enabled fast measurement of a large number of process variables, resulting in large data streams, and, second, conventional multivariate methods such
as Hotelling's $T^{2}$, MEWMA, and MCUSUM (~\cite{wierda1994multivariate,sparks1992quality})
are not scalable in terms of the computational time and the
detection power. 

An example of large streams can be found in gas turbine
systems used for power generation. In these systems, the performance
of the confined combustion process is being monitored using hundreds of
sensors measuring temperature, vibration, pressure, etc., in different
chambers and segments of the turbine. Early detection of any changes
in the system, followed by the diagnosis of the faulty variables is
necessary to avoid imminent blowout that leads to relighting the combustor and costly shutdowns. 

Another application of large streaming data is in image-based process monitoring in which each pixel of an image
can be considered as a single data stream. For example, in a rolling
process where a set of rollers are used to reduce the cross-section
of a long steel bar by applying compressive forces, the quality of
produced bars can be inspected by a vision system that is set up to
take images of the bar surface at short time intervals. A sample of
such an image is shown in Figure~\ref{fig:rolling}. In this image, each row contains 512 pixels, and each pixel can be considered as a variable resulting in an high-dimensional correlated data stream.

\begin{figure}[]
\centering{}
\includegraphics[width=0.4\linewidth]{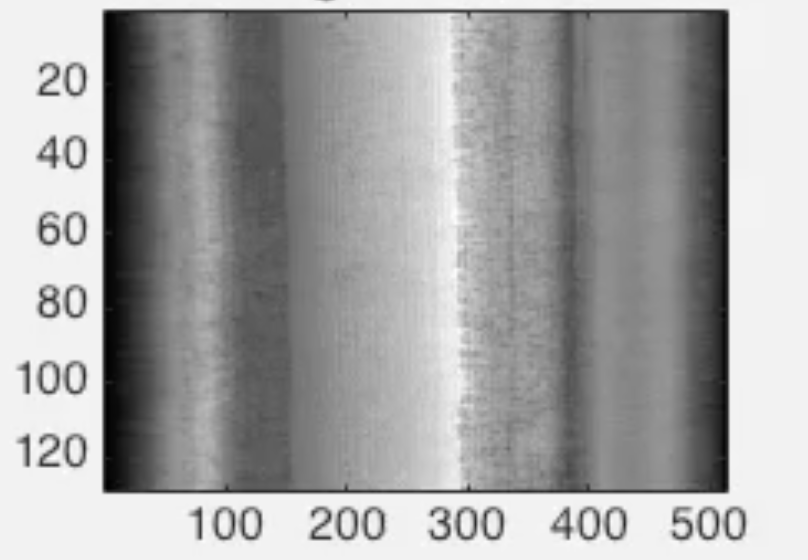}
\caption{Surface image of steel bar in the rolling process }
\label{fig:rolling}
\end{figure}

Despite its importance, existing SPC literature lacks a scalable integrated
M\&D approach using large data streams. In the following, we will discuss the
existing work on \emph{monitoring} and \emph{diagnostics}, their shortcomings---especially 
their lack of integrability---, and motivate our approach.

\subsubsection*{Monitoring}

Conventional multivariate monitoring charts are effective on small or moderate data streams. However, their performance deteriorates as the number of data streams increases. 
To address the high-dimensionality issue, more recent works have focused
on employing variable selection techniques to reduce the dimensionality
by removing the variables that are less susceptible to the process
change. Examples of the variable-selection-based method include \cite{wang2009high,zou2009multivariate,capizzi2011least}.
However, these methods are not scalable and generally require intensive
 computation as the dimension grows. Moreover, most of
these methods are difficult to interpret for a process
engineer. 

There is a section of scalable multivariate monitoring methods that are
developed based on the assumption that data streams are independent
and the change is sparse (only a small subset of variables is affected by a process change). 
For example, \cite{tartakovsky2006detection}
assumed that exactly one variable changes at a time and
proposed an approach based on the maximum of CUSUM statistics from
each individual data stream. \cite{mei2010efficient,mei2011quickest}
developed robust monitoring scheme based on the sum of (the top-r)
local CUSUM statistics assuming all variables are independent and
measurable. 

For the case that variables are not easily or efficiently
measurable, \cite{liu2015adaptive} presented TRAS (top-r based adaptive
sampling), which is an adaptive sampling strategy that uses the sum
of top \textit{r} local CUSUM statistics for monitoring. The sparsity
assumption is generally valid in practice as a change or fault often
affects only a small subset of variables. However, although theoretically
and computationally appealing, the independence assumption is typically
unrealistic.

To address the dependency and high-dimensionality issues, Principal
Component Analysis (PCA) has been widely used for monitoring multivariate
data streams. PCA is a well-known projection technique that transforms
dependent data to uncorrelated features known as Principal Component
(PC) scores. Often, only a few PC scores that explain the most variation
of original data are used for monitoring in a dimension
reduction (~\cite{jackson1979control,wise1990theoretical,qahtan2015pca,li2000recursive,li2014new}).
However, monitoring top PCs with the highest variance may not always be a right approach.

As an example, consider a bivariate normal distribution given in Figure~\ref{fig:PCA_2D}, 
in which $PC_1$ represents the direction of the eigenvector
with larger eigenvalue, and the red arrow indicates the direction
of change in the mean of the distribution. As can be seen from the
figure, the effect of the change on both PC-scores is the same. However,
the fact that PC1 constitutes the most of the process variance makes
it less sensitive to the change compared with PC2. In other words,
small changes may be masked by the variation present in top PC scores,
hence becoming undetectable. 

\begin{figure}[]
\centering{}
\includegraphics[width=0.8\linewidth]{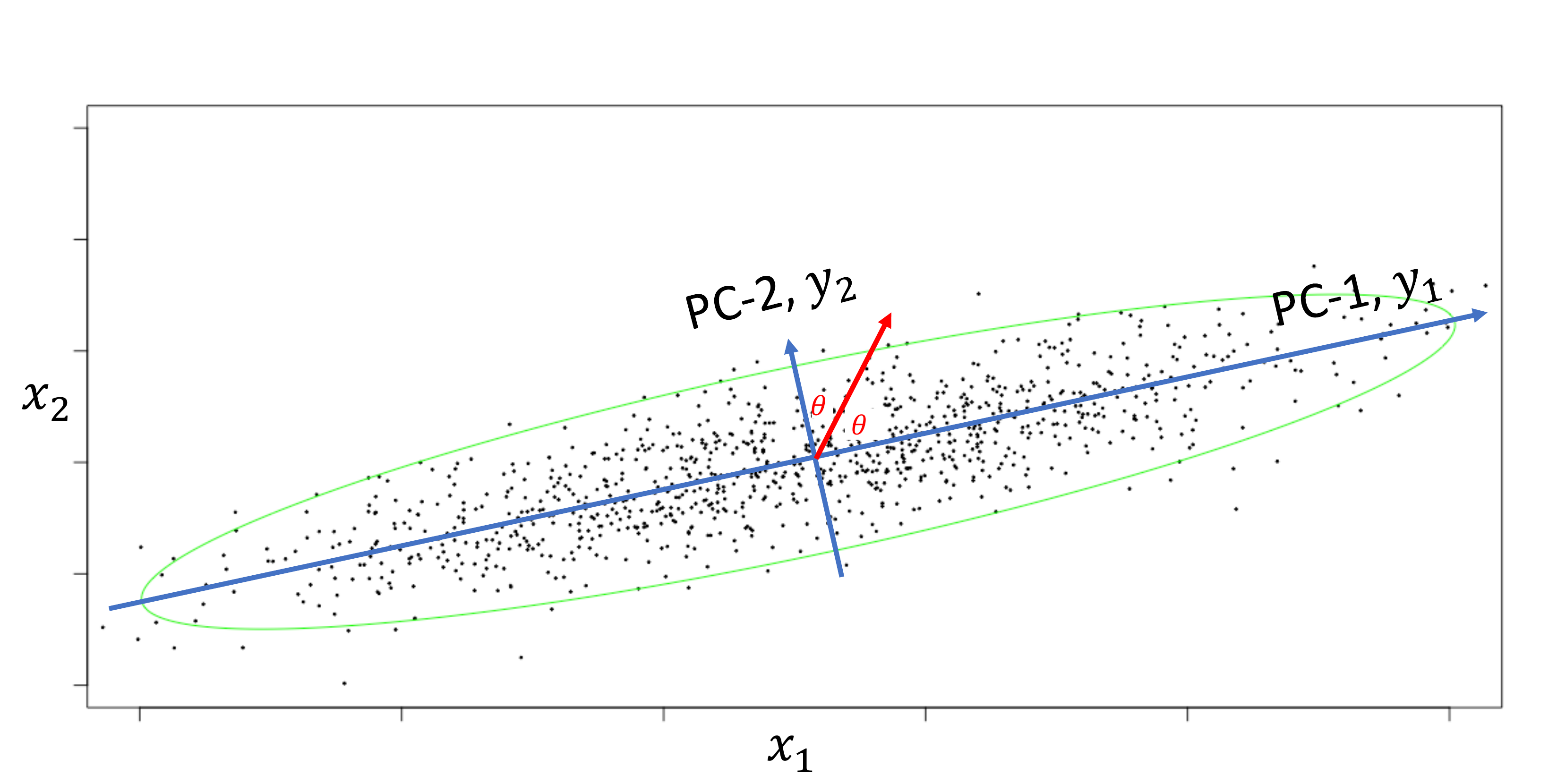}
\caption{Example of a change having the same angle with both PCs }
\label{fig:PCA_2D}
\end{figure}

Other PC selection criteria for process
monitoring include the variance of reconstruction error (VRE) approach
by \cite{dunia1998subspace}, and the fault signal to noise ratio
(SNR) by \cite{yuan2009improved} and \cite{tamura2007study}. The
VRE method selects a subset of PCs which minimizes fault reconstruction
error, while the fault SNR method is based on fault detection sensitivity.
The main disadvantage of these methods is that they require prior knowledge of the fault direction. 

In this paper, we present an adaptive
PC Selection (APC) approach based on hard-thresholding for selecting
the set of PC scores that are most susceptible to an unknown change.
Unlike the top-r-PCs, in our approach the number of features
may vary at each sampling time. Also the PCs are adaptively selected based on the observed sample and its standardized distance from the in-control mean. 
Additionally, the proposed APC approach does not require any prior knowledge
about a fault or change direction, which makes it more universally
applicable.

\subsubsection*{Diagnostics}

Another long-standing issue with PCA-based monitoring methods is the
lack of diagnosability. This is because that the PC scores used as
monitoring features are linear combinations of original measurements. Therefore, if a PC score initiates an out-of-control alarm, it is difficult to attribute it to any specific process variable(s). Interpretation and decomposition of these additive statistics are often theoretically difficult and/or computationally expensive
in high-dimensional data streams. 

For diagnostics on a PC-based monitoring, one
common approach is the use of contribution plots that specifies the
contribution of each variable to the out of control statistic ~\cite{westerhuis2000generalized,alcala2009reconstruction,joe2003statistical,qin2001unifying}.
Contribution plots are popular because of their ease of implementation
and their ability to work without any a priori knowledge. However,
correct isolation with contribution plots is not guaranteed for multiple
sensor faults ~\cite{yue2001reconstruction}. 

To overcome this problem, hierarchical contribution plots was proposed ~\cite{macgregor1994process}.
However, it will perform poorly if the initial partitioning is not correct.
Moreover, in the context of high-dimensional data, these methods become difficult
to interpret and are also computationally expensive.

For the purpose of diagnosis in multivariate control charts with original
measurements, \cite{wang2009high} and ~\cite{zou2009multivariate}
proposed variable selection techniques. Both methods optimize a penalized
likelihood function for multivariate normal observations to identify
the subset of altered variables. The $\mathcal{L}_1$-penalized regression
method of \cite{zou2009multivariate} provides more computational
advantages in implementation. \cite{zou2011lasso} combined Bayesian
Information Criterion (BIC) with penalization techniques to assist
the fault localization process and suggested an Adaptive Lasso-based
diagnostic procedure. However, these methods assume that the change
point is already known and focus only on diagnosis. 

Additionally,
they cannot easily be integrated with a PCA-based monitoring approach.
To address these shortcomings, we propose a new diagnostics approach
that seamlessly integrates with our proposed PCA-based monitoring
method. The developed approach draws inspiration from Compressed Sensing and uses Adaptive Lasso to identify the shifted variables.
In this paper, we focus on detecting mean shifts and we assume
the shift is sparse. As mentioned earlier, this is a reasonable assumption
because in real-world usually only a small number of variables change.

The major contributions in this paper are, a. countering the traditional
view of top-PC scores as the best for process monitoring, and
proposing an adaptive PC selection approach as an alternative; and
b) proposing a new diagnostics approach that integrates with the proposed
PCA-based monitoring framework. %\todo[inline]{and showing it's consistency}.
An overview of the proposed Monitoring and Diagnostics (M\&D) approach
is shown in Figure~\ref{fig:Methodology-overview}.

\begin{figure}[]
\centering{}
\includegraphics[width=0.6\linewidth]{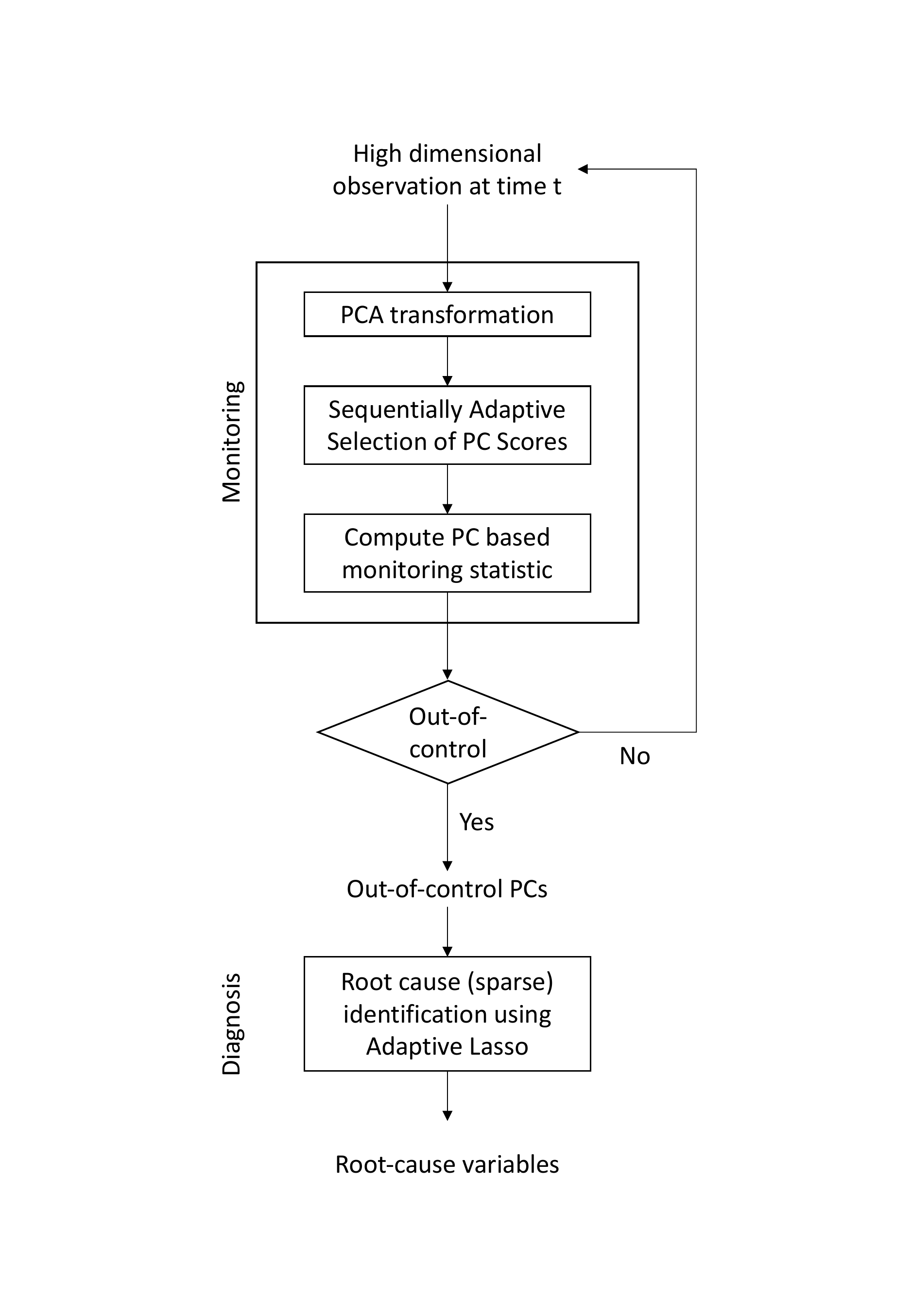}
\caption{Methodology overview}
\label{fig:Methodology-overview}
\end{figure}

% The rest of the paper is as organized as follows. Next section provides
% an overview of the proposed integrated monitoring-diagnostics framework.
% Then, the novel APC selection approach is presented and is integrated
% with an EWMA control chart. After that our proposed diagnostics approach
% based on adaptive lasso on PC scores is elaborated. The next section
% will present simulation studies for different scenarios, and performance
% analysis of our method in comparison with a few existing methods as
% benchmarks. Then, in the consequent section, using two case studies,
% we show that the proposed methods significantly outperform the benchmarks in terms of quick change detection (monitoring) and identification of the changed variables (diagnostics). Finally, we conclude and provide future research directions.

\section{Integrated PCA-Based Monitoring and Diagnostics\label{sec:Proposed-Monitoring-Diagnostics-}}

% \section{Adaptive PC selection approaches for process monitoring\label{subsec:Novel-PCA-selection-approaches}}

\subsection{Background\label{subsec:Background-Monitoring}}

PCA is a linear transformation widely used for dimension reduction
and generating uncorrelated features. Suppose 
a $p$-dimensional data stream denoted as $X=\{\mathbf{x}^{(t)}:\mathbf{x}^{(t)}\in\mathbb{R}^{p};\ t=1,2,\ldots\}$
is collected at sampling time $t$. Without loss of generality, assume
the data streams are centered (zero mean) with a covariance matrix $\boldsymbol{\Sigma}$.
By applying PCA, this set of correlated observations can be converted
into a set of linearly uncorrelated variables known as \textit{principal
component scores}. The PC scores can be computed by $\mathbf{y}=A^{T}\mathbf{x}$,
where $A\in\mathbb{R}^{p\times p}$ is the matrix of eigenvectors
of $\boldsymbol{\Sigma}$ and $\mathbf{y}\in\mathbb{R}^{p}$ are the PCs.
Also, it can be shown that $var(\mathbf{y}_{j})=\lambda_{j}$ and
$cov(\mathbf{y}_{j},\mathbf{y}_{k})=0,\ \forall(j,k)\in\{1,\dots,p\},\,j\neq k$.

For ease of interpretation, the eigenvectors in $A$ are arranged
such that their corresponding eigenvalues are in decreasing order,
i.e. $\lambda_{1}\geq\lambda_{2}\geq\ldots\geq\lambda_{p}$. This
ordering will be further referred throughout the paper for developing
our methodology. In this paper, we call the principal scores corresponding
to higher and lower eigenvalues as top-PCs and bottom-PCs, respectively.

In most conventional PCA-based methods, top-$k$ PCs are selected
for monitoring because they contain more process information. This
approach, however, may not always result in an appropriate set of monitoring variables.
To illustrate this, we synthesized in-control samples of correlated
data from a multivariate normal distribution with a dimension of 500 ($=p$) and $\bm{\mu}=\bm{0}$, $\sigma=0.1$ followed by out-of-control samples. In the out-of-control data, a random 10\% set of variables are shifted by $0.05\sigma$. We perform PCA on the data and
monitor all PCs separately.

Figure~\ref{fig:normplot_resid-1} shows
the control charts for top 5 and bottom 5 of PCs. As shown in the
figure (left), the top PCs fail to detect the change. As mentioned
earlier, this is due to the fact that top PCs have large variances,
which make them insensitive to small shifts in the mean. On the other
hand, the bottom PCs with smaller variances are more sensitive and
can detect the small shift at the time it occurs, i.e., $t=50$. This
experiment was repeated several times, and each time similar results
were found. 

This shows that depending on the direction and the size of a
change, the traditional approach of selecting top PCs may severely
underperform. Therefore, it is imperative to develop a PC selection
approach that can adaptively select the set of most sensitive PCs and does not depend on the a priori knowledge about the direction
of the change.

\begin{figure}
\centering
\subfloat[top 5 PCs]{\includegraphics[width=0.37\paperwidth]{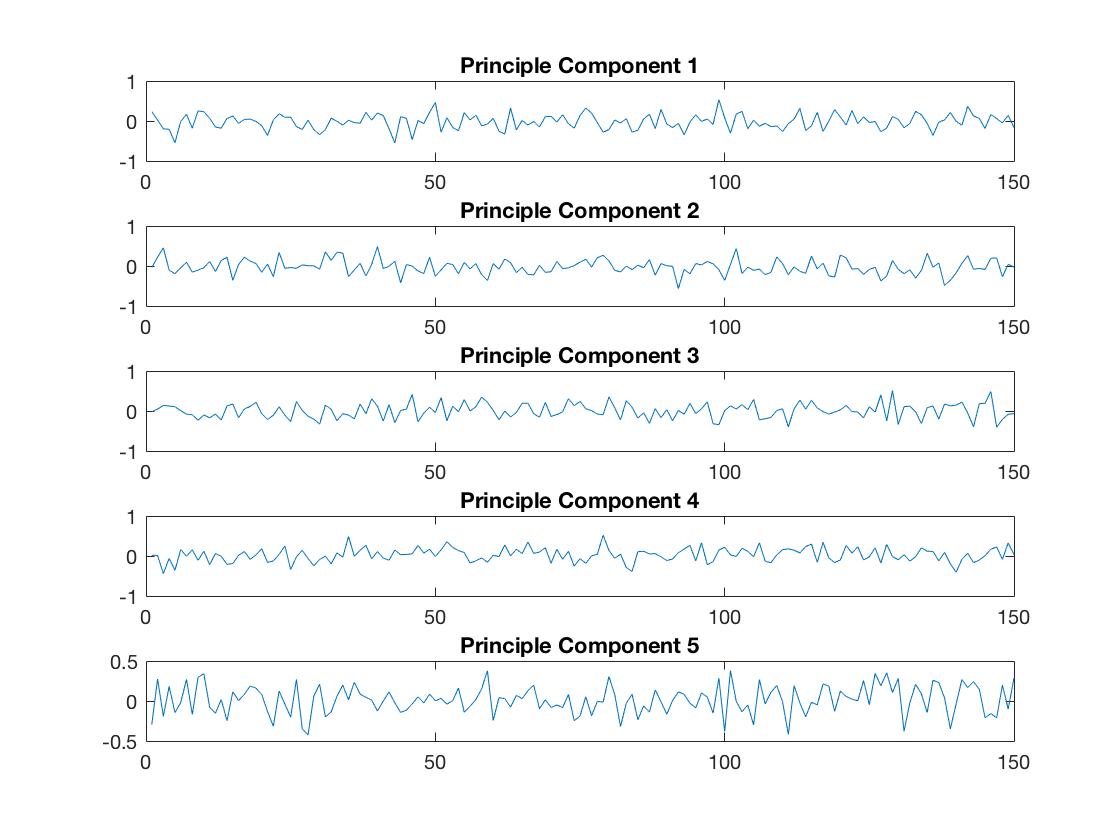}}
%\hspace{1ex}
\subfloat[low 5 PCs]{\includegraphics[width=0.37\paperwidth]{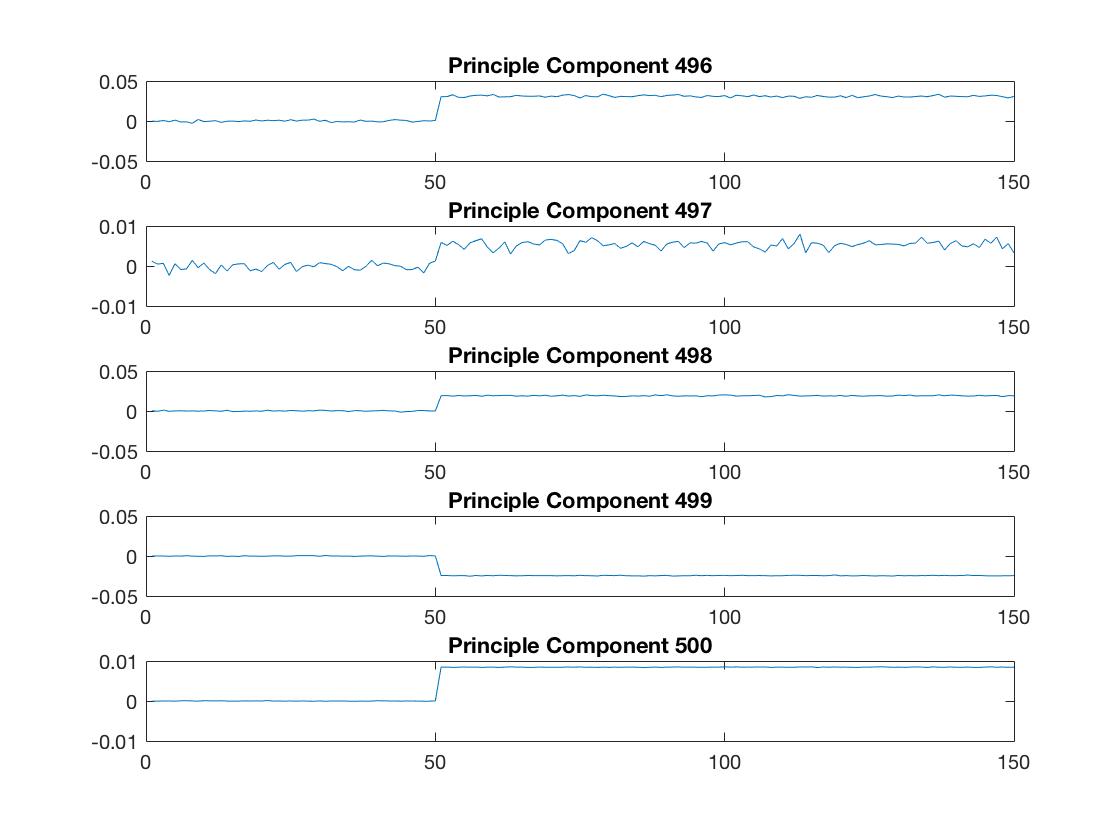}}
\caption{Comparing the behavior of top and low PCs for monitoring when a sparse shift happens in a random set of process of variables}\label{fig:normplot_resid-1}
\end{figure}

\subsection{Adaptive PC selection (APC) for Process Monitoring \label{subsec:Method-III:-sequential PCs}}

In this section, we propose an adaptive PC Selection approach for
process monitoring. This approach selects and monitors a set
of PCs that show a higher deviation from a known in-control state.
Suppose, the in-control observations follow $\mathbf{\mathbf{x}_{\mathit{t}}}\sim N(\mathbf{0},\boldsymbol{\Sigma});t<\tau$,
and at an unknown time $\tau$ a mean shift occurs such that $\mathbf{\mathbf{x}_{\mathit{t}}}\sim N(\pmb{\mathbf{\mu}},\boldsymbol{\Sigma});t\geqslant\tau$,
where, $\pmb{\mu}$ is a non-zero sparse vector, and the process covariance
is assumed to remain constant. Therefore, given the eigenvector matrix
$A\in\mathbb{R}^{p\times p}$, the PC scores after the process change
will become $\bm{\mathbf{y}_{\mathit{t}}}=A^{T}\mathbf{x}_{\mathit{t}}\sim N(A^{T}\bm{\mu},\Lambda)$,
where $\Lambda=diag(\lambda_{1},\lambda_{2},\dots\lambda_{p})$. 

The
standardized expected shift magnitude along the $j^{th}$ PC can be
obtained by $\delta_{j}=\frac{||\pmb{\mu}||\cos\theta_{j}}{\sqrt{\lambda_{j}}};j=1,2,...,p$,
where $\theta_{j}$ is the angle between the shift direction and the
$j^{th}$ PC. As can also be seen from Figure~\ref{fig:PCA_2D}, this can imply that
a PC closer to the shift direction (i.e. smaller $\theta$) will capture
a larger shift magnitude. Moreover, if $\theta$ is similar for two
PCs, then the one with the smaller variance will be more sensitive
to the change. Therefore, to take both measures into account, we work
with standardized PC score, denoted by $\tilde{y}_{tj}=\cfrac{y_{tj}}{\sqrt{\lambda_{j}}}$,
that contains both magnitude and sensitivity information.

% ADD ALG 1 and 2
\iffalse 
\begin{figure}
\noindent \begin{centering}
\noindent\begin{minipage}[t]{1\columnwidth}%
Algorithm: Preprocessing based on in control data: in-control data,
- calculate matrix of PC vectors, A, based on in-control observation, and standard deviation of PCs as - Using simulation obtain the as
the control limit of the monitoring statistics, such that the in-control
Average Run Length over 10,000 iteration is equal to our value. Algorithm:
At time t, observe new vector of observation Normalize the vector
x (if needed, based on in control data) Calculate vector of PCs y:
Standardize each PC: For each standardized PCs calculate the EWMA statistic, where Standardize all EWMA statistics, and calculate
the squared transformation Perform a thresholding on and calculate
the monitoring statistic as the sum of thresholded variables: where
Check the control limit of the if raise an alarm %
\end{minipage}
\par\end{centering}
\caption{APC algorithm\label{fig:sAPC-algorithm}}
\end{figure}

\fi

% Suppose, we have a process with an in-control distribution as, $\bm{x}\sim(\bm{0},\Sigma$), and after a shift at any time $t$ it becomes,
We choose the EWMA statistic for monitoring as it is more sensitive
to small changes and includes the information of previous samples.
The EWMA statistic, denoted by $z_{tj}$, is defined as ${z}_{tj}=\gamma\tilde{y}_{tj}+(1-\gamma){z}_{(t-1)j};\;t=1,2,...;\;j=1,2,...,p$,
where ${z}_{0}=0$, and $\gamma\in[0,1]$ is a weight. Under the in-control
process, ${z}_{tj}\sim N\left(0,\sigma^{2}=\frac{\gamma}{1-\gamma}\right)$.
Consequently, its squared standardized value follows a Chi-squared
distribution with one degree of freedom, i.e., $d_{tj}=(\frac{z_{tj}}{\sqrt{\frac{\gamma}{1-\gamma}}})^{2}\sim\chi_{(1)}^{2}$.

When the process is out-of-control, depending on the direction of
the mean shift, a few $d_{tj}$ values will inflate, while the rest
are slightly affected (or unaffected) by the mean shift. To increase
the detection power of the monitoring procedure, these PCs should
be filtered out. For this purpose, following \cite{wang2013montoring}, we use a soft-thresholding operator to define the following aggregated monitoring statistics,

\begin{equation}
R_{t}=\sum_{j=1}^{p}(d_{tj}-\nu)_{+},\label{eq:R-Method-3}
\end{equation}

where the operator $(\cdot)_{+}=\max\left\{ 0,\cdot\right\} $, and
$\nu$ is the threshold value selected based on a desired significance
level of $\chi^{2}$ test. We monitor the $R_{t}$ statistic and raise
an alarm if, $R_{t}>R_{0}$, where $R_{0}$ is the threshold level
found for a desired in-control ARL using simulations.

\noindent \textbf{\emph{Selection of Control Limit ($R_{0}$)}}

To determine an appropriate value of $R_{0}$, we require the
distribution of monitoring statistic $R$. To find the distribution, we first specify the moments
of thresholded values in Proposition~\ref{prop:d_tilde}.

\begin{proposition} 
\label{prop:d_tilde} 
If $d_{tj}\sim\chi_{1}^{2}$, then their soft thresholded values $\tilde{d}_{tj}=(d_{tj}-\nu)_{+}$ follows a bimodal truncated $\chi_{1}^{2}$ distribution, with the following moments

$E(\tilde{d}_{tj})=E((d_{tj}-\nu)_{+})=\frac{1}{\Gamma(0.5)}[\Gamma(0.5,\frac{\nu}{2}+e^{-\frac{\nu}{2}}\sqrt{2\nu}]-\nu P(\chi_{1}^{2}>\nu)$

$E(\tilde{d}^{2}_{tj})=E((d_{tj}-\nu)_{+}^{2})=\frac{1}{\Gamma(0.5)}[3\Gamma(0.5,\frac{\nu}{2}+e^{-\frac{\nu}{2}}\sqrt{2\nu}(3+\nu)]-2\nu E(\tilde{d})-\nu^{2}P(\chi_{1}^{2}>\nu)$
\end{proposition}

Proof is provided in Appendix~\ref{Ap:prop_dtilde_prove}. %\se{write down the derivations}

Using the Central Limit Theorem, $R ~ N(p\mu_{\tilde{d}}, \sqrt{p}\sigma_{\tilde{d}})$.
Hence, $R_{0}$ for a desired type I error, $\alpha$, is
\begin{equation}
R_{0}=p\mu_{\tilde{d}}+\sqrt{p}\sigma_{\tilde{d}}\Phi^{-1}(1-\alpha),\label{eq:R0_choice}
\end{equation}

\noindent where $\Phi$ is the inverse normal cdf. To validate this approach and the normal
approximation, we perform simulations in Section~\ref{sec:sim_valid_thr}.
The results show that the empirical $\alpha$ obtained by this approach
is very close to the true $\alpha$.

\subsection{PC-based Signal Recovery (PCSR) Diagnosis Methodology\label{subsec:Diagnostics-Methodology}}

In monitoring high-dimensional data streams, apart from quick detection of changes,
precise fault diagnosis to identify accountable variables is extremely
crucial. Diagnosis aims at isolating
the shifted variables, which will help identify and eliminate the root
causes of a problem. However, despite its importance, very
few diagnostic methods exist for high-dimensional data streams that is 
integrable with a PCA-based monitoring.

To that end, we propose a diagnostics approach that seamlessly
integrates with the proposed PCA-based monitoring for large data streams.
Inspired by Compressed Sensing (CS), we develop an adaptive lasso
formulation to recover the variables responsible for an out-of-control
alarm. We assume that only the process mean has shifted and the
shift is sparse. 

In CS, a high-dimensional sparse original signal can be reconstructed
from noisy transformed observations by finding solutions to an \foreignlanguage{english}{underdetermined
linear system}. In other words, given a set of observations $\mathbf{y}$,
and a transformation (sensing) matrix $\Upsilon$, a sparse unknown
original signal $\pmb{\mu}$ can be recovered from $\mathbf{y}=\Upsilon\pmb{\mu}+\pmb{\epsilon}$,
where $\pmb{\epsilon}$ denotes the random errors. 

The outcome of a PC monitoring method can be formulated similarly
to identify the shifted process variables. Without loss of generality,
we suppose the process has mean $\mathbf{0}$ during in-control that
changes to a sparse mean $\pmb\mu$ during the out-of-control of state.
Therefore, the out-of-control observations follow $\mathbf{x}=\pmb{\mu}+\pmb{\epsilon}$,
where $\pmb{\epsilon} \sim(\bm{0},\Sigma)$. Consequently, the out-of-control
PC scores are,

\begin{equation}
\mathbf{y}=A\mathbf{x}=A\pmb{\mu}+\tilde{\pmb{\epsilon}},\label{eq:y=00003D00003DAx}
\end{equation}

\noindent where, $\tilde{\pmb{\epsilon}}=A\pmb{\epsilon}$ is the noise in the
PC domain, with zero mean and covariance of $\Lambda$= diag($\lambda_{1},\lambda_{2},\dots\lambda_{p}$).

Looking at Eq. \ref{eq:y=00003D00003DAx}, we can notice its similarity
with a compressed sensing problem. In Eq. \ref{eq:y=00003D00003DAx},
the eigenvector matrix, A, and the principal scores, $\bm{y}$, are
known, and we wish to estimate the shifted mean $\mu^ {}$ when an
out-of-control situation is detected after monitoring. \cite{candes2005decoding}
and~\cite{haupt2006signal} showed that a least squares objective
function with $L_{1}$ penalty, also known as lasso, can be used to
estimate the sparse vector $\pmb\mu$. Since lasso estimates in general
are not consistent, we use adaptive lasso \cite{zou2006adaptive}
to build our diagnosis model. Specifically, 

\begin{equation}
\hat{\pmb{\mu}}=\operatornamewithlimits{argmin}_{\pmb{\mu}}\{||\mathbf{y}-A\mathbf{\bm{\mu}}||_{l_{2}}^{2}+r\sum_{j=1}^{p}w_{j}|\mu_{j}|,\label{eq:weightedL1-1}
\end{equation}

\noindent where $r$ is a nonnegative regularization parameter and $\textbf{w}=\frac{1}{\hat{\mu}_{OLS}}$ is the data dependent weight vector.

%\todo[inline]{change the wording}
One problem in solving Eq.\ref{eq:weightedL1-1} is that the covariance matrix
of $\pmb{\epsilon}'$ is not homogeneous. The variance heterogeneity may affect
the estimation performance. To address this issue, we apply the following
transformation to get constant variances for all error terms. %\todo[inline]{cite weighted regression} 
\begin{equation}
\bm{\ y}^{*}=\Lambda^{-\frac{1}{2}}\bm{y},\ \ \ \ \ A^{*}=\Lambda^{-\frac{1}{2}}A,\ \ \ \ \pmb{\epsilon}^{*}=\Lambda^{-\frac{1}{2}}\pmb{\epsilon}^{\prime}.\label{eq:lasso_transform}
\end{equation}

Consequently Eq.~\ref{eq:y=00003D00003DAx} is transformed to $\mathbf{y}^{*}=A^{*}\pmb{\mu}+\pmb{\epsilon}^{*}$,
where $\pmb{\epsilon}^{*}\sim(\textbf{0},\textbf{I})$ with $\textbf{I}$
as a $p$ dimensional identity matrix. The updated adaptive lasso
formulation is given by.

\begin{equation}
\hat{\pmb{\mu}}=\operatornamewithlimits{argmin}_{\pmb{\mu}}\{||\mathbf{y^{*}}-A^{*}\mathbf{\bm{\mu}}||_{l_{2}}^{2}+r\sum_{j=1}^{p}w_{j}|\mu_{j}|\}
\label{eq:weightedL1}
\end{equation}

\noindent Where $w_{j}=1/|\hat{\mu}_{j}|$, and $\hat{\pmb\mu}$ is a
root p-consistent estimator to $\pmb\mu$, e.g., $\hat{\pmb\mu}=\hat{\pmb{\mu}}_{ols}$.
The optimization problem in Eq. \ref{eq:weightedL1} can be solved
using various optimization algorithms, such as gradient descent, proximal
descent, and LARS. In our implementation, we used the gradient descent
method. 

After finding the solution, the set of variables whose corresponding
estimated $\mu_{j}$ is non-zero is considered as the altered variables.
It should be noted that according to Theorem~2 in \cite{zou2006adaptive},
the estimated mean is consistent, loosely meaning that when $\textbf{A}^{*}$
has large dimension (i.e. large $p$), the non-zero components of
$\pmb\mu$ are correctly identified. This implies that the larger the
number of data streams the higher is the likelihood of correct
diagnosis. See Appendix~\ref{Ap:consistency} for more details. 

To determine the value of parameter $r$, one can use Bayesian
Information Criterion (BIC) (\cite{schwarz1978estimating}) and choose
the $r$ value that results in the smallest BIC value. The reason
behind choosing BIC is that it can determine the true sparse model
if the true model is included in the candidate set (\cite{yang2005can}).
Since, in the diagnosis problem the objective is to detect the nonzero
elements (shifted variables) rather than estimation of the out-of-control
mean, BIC is a proper criterion for diagnosis \cite{zou2011lasso}. 

\section{Experimental Analysis \label{sec:Experimental-Analysis}}
%\ref{theo:R_dist}
In this section, first we validate Proposition \ref{prop:d_tilde} using
simulations. Afterwards, we study the performance of the proposed
monitoring-diagnostic method in change detection and in terms of quick
detection of mean shifts and identification of altered variables.
For all the experiments, we simulate data streams such that in-control data follows a multivariate normal distribution $~N(\mathbf{0},\Sigma)$ and the out-of-control is $N(\pmb{\mu}_{1},\Sigma)$, $\pmb{\mu}_{1}$ is sparse.
We carry out the simulations for different levels and types of shifts and the covariance structure and compare the results with existing methods as benchmarks.

\subsection{Validation of Proposition \ref{prop:d_tilde} for Choosing Control Limits}
\label{sec:sim_valid_thr} 

To validate Proposition \ref{prop:d_tilde},
we perform two sets of experiments. In the first experiment, we generate
$d_{j}\sim\chi_{1}^{2}$ for $j=1,\cdots,p$, and $R_{t}=\sum_{j=1}^{p}(d_{tj}-\nu)_{+}$
for $t=1,\cdots,1000$, similar to Eq. \ref{eq:R-Method-3}. Then
we calculate $R_{0}$ using Eq.~\ref{eq:R0_choice} for desired $\alpha=0.05$.
We calculate the empirical Type I error, denoted by $\tilde{\alpha}$,
as the fraction of times $R_{t}$'s pass the control limit $R_{0}$.
We perform this experiment for different values of $p$ and $\nu$
and replicate each scenario 1000 times. Finally, we report the average
empirical Type I errors in table \ref{tab:emp_type1_ex1}.

In the second experiment, first we simulate $\textbf{x}_{t}$ for
$t=1,\cdots,1000$ as a $p$ dimensional normal distribution random
variables with random covariance matrix and zero mean. Given the eigenvector
matrix $\textbf{A}$, we calculate its PC scores, and its
corresponding EWMA statistic using $\gamma=0.4$. 
Consequently, its squared standardized value are calculated as $d_{tj}=(\frac{z_{tj}}{\sqrt{\frac{\gamma}{1-\gamma}}})^{2}$.

Here, for each observation we define $R_{t}=\sum_{j=1}^{p}(d_{tj}-\nu)_{+}$,
we repeat this procedure 1000 times. Similar to previous experiment,
we calculate $R_{0}$ using Eq.~\ref{eq:R0_choice} for desired $\alpha=0.05$.
We calculate the empirical Type I error, $\tilde{\alpha}$, as the
fraction of times $R_{t}$'s pass the control limit $R_{0}$. We replicate
each $(p, \nu)$ scenario 1000 times and we report the average empirical
type I errors in table \ref{tab:emp_type1_ex2}.

As can be seen from Table~\ref{tab:emp_type1_ex1}-\ref{tab:emp_type1_ex2}, 
as $p$ increases, the empirical
Type I error approaches to its true value $\alpha=0.05$. Moreover,
for large $p$, the result is less sensitive to the choice of the
threshold value, $\nu$. Hence, it shows the validity of the proposed
approach for finding control limits. 

Note that the main difference
between these studies is the independence of $R_{t}$'s. In the first
study, $R_{t}$'s are independently generated, whereas in the second
study, $R_{t}$'s are calculated using EWMA statistics, which are not
independent. The larger bias in the results of the second study is
mainly because the monitoring statistic values are autocorrelated.
However, for very large $p$ (e.g, $p>5000$), this difference is
negligible. For smaller $p$, we would suggest using a Monte Carlo
simulation to determine the control limits.

\begin{table}[h]
\centering \caption{Empirical type I error using first experiment}
\label{tab:emp_type1_ex1} %
\begin{tabular}{cllllll}
\multicolumn{1}{l}{} &  & \multicolumn{5}{c}{p}\tabularnewline
\multicolumn{1}{l}{} & \multicolumn{1}{l}{} & 100  & 500  & 1000  & 5000  & 10000 \tabularnewline
\hline 
\multirow{6}{*}{$\nu$}  & \multicolumn{1}{l}{0.05} & 0.0090  & 0.0067  & 0.0063  & 0.0056  & 0.0053 \tabularnewline
 & \multicolumn{1}{l}{0.10} & 0.0091  & 0.0067  & 0.0060  & 0.0055  & 0.0052 \tabularnewline
 & \multicolumn{1}{l}{0.15} & 0.0090  & 0.0067  & 0.0062  & 0.0056  & 0.0054 \tabularnewline
 & \multicolumn{1}{l}{0.20} & 0.0090  & 0.0066  & 0.0064  & 0.0055  & 0.0054 \tabularnewline
 & \multicolumn{1}{l}{0.25} & 0.0091  & 0.0068  & 0.0061  & 0.0055  & 0.0054 \tabularnewline
 & \multicolumn{1}{l}{0.35} & 0.0092  & 0.0068  & 0.0063  & 0.0055  & 0.0053 \tabularnewline
\end{tabular}
\end{table}

\begin{table}[h]
\centering \caption{Empirical type I error using second experiment}
\label{tab:emp_type1_ex2} %
\begin{tabular}{cllllll}
\multicolumn{1}{l}{} &  & \multicolumn{5}{c}{p}\tabularnewline
\multicolumn{1}{l}{} & \multicolumn{1}{l}{} & 100  & 500  & 1000  & 5000  & 10000 \tabularnewline
\hline 
\multirow{6}{*}{$\nu$}  & \multicolumn{1}{l}{0.05} & 0.0091  & 0.0071  & 0.0068  & 0.0065  & 0.0050 \tabularnewline
 & \multicolumn{1}{l}{0.10} & 0.0091  & 0.0073  & 0.0067  & 0.0064  & 0.0050 \tabularnewline
 & \multicolumn{1}{l}{0.15} & 0.0090  & 0.0073  & 0.0067  & 0.0063  & 0.0050 \tabularnewline
 & \multicolumn{1}{l}{0.20} & 0.0091  & 0.0072  & 0.0067  & 0.0063  & 0.0051 \tabularnewline
 & \multicolumn{1}{l}{0.25} & 0.0089  & 0.0072  & 0.0068  & 0.0061  & 0.0049 \tabularnewline
 & \multicolumn{1}{l}{0.35} & 0.0083  & 0.0066  & 0.0063  & 0.0057  & 0.0047 \tabularnewline
\end{tabular}
\end{table}

\subsection{Monitoring Methods Analysis\label{subsec:Monitoring-Methods-Analysis}}

In this section, we conduct various simulations to validate the performance
of the proposed monitoring method based on the Average Run Length
(ARL) and its standard error for different magnitudes of shifts. Specifically,
the following scenarios are considered: 
\begin{description}
\item [{I.}] Random covariance structure and random shift: To generate
the random covariance matrix, we use the Wishart distribution with
diagonal entries equal to 1. To generate a sparse mean shift, we randomly
select 20\% of the process variables and shift them by $\delta$. 
\item [{II.}] Block diagonal covariance: This scenario mimics the situations
where each data stream is correlated with only a subset of the data
streams. The covariance matrix used in this scenario has $K=12$ blocks,
denoted as $B_{k}$, $k=1,...,K$. Each block $B_{k}$ is a
random semi-positive definite matrix generated from a Wishart distribution.
To generate out-of-control data, we shift the mean of some of the variables
that belong to only one of the blocks ($B_{k}$), by $\delta$ , i.e.,
$\mu{}_{j}=\begin{cases}
\begin{array}{c}
\delta\ j\in B_{k}\\
0\ j\notin B_{k}
\end{array}\end{cases}$. 
%Besides, we select the block variables which are not represented in either the top first or last principal component scores. Note that the change generated in this scenario is difficult to detect since not only is the mean shift in the direction of one of the blocks, but also it is not captured by the top or low PCs. %\item [{III)}] Shift in the direction of top PCs: In this case, the covariance matrix is of random form and shift occurs in the direction of principal
% component vectors with the highest variation (top PCs). Therefore, the out-of-control mean, $\mu_{1}$, will be of the linear
% combination of top eigenvectors, i.e. $\mu_{1}=\delta\times[\sum_{k=1}^{K}A_{K}]$.
% Where $A_{i}$ is the $i^{th}$ eigenvector. Note that in this scenario the shift is not sparse, but
% its transformation in the PC direction will become sparse. Hence,
% after the change, the PCs will be $y_{i}=\delta$ for $i=1\dots K$
% and $y_{i}=0$ for $i>K$. Also, note that this scenario again is
% an extreme case that the shift is exactly in the direction of top PC scores
% and is not represented in any of lower PCs.
\end{description}
In each scenario, $p$ data streams are generated. We run the simulations for, $p= 100$, $p= 1000$, and $p= 10,000$ to evaluate the performances in different dimensions.
We apply the proposed monitoring method, APC, and compare it with
three existing methods: 

a) $T_{new}$ by \cite{zou2015efficient}.
This monitoring method is based on a goodness-of-fit test of the local
CUSUM statistics from each data stream. 

b) TRAS by \cite{liu2015adaptive}.
Top-r based adaptive sampling (TRAS) is an adaptive sampling strategy
that uses the sum of top-r local statistics for monitoring. Since
this method works only for independent variables, we will implement
it on PCs rather than original data, %c) FDR (False Discovery Rate) on PCs. \cite{benjamini1995controlling} proposed a simple linear step-up procedure aimed at controlling the FDR at a pre-specified significance level while maximizing the number of rejected hypotheses. We will implement this procedure on standardized PCs to detect the out of control process. The final method we wish to compare with ours is 

c) Traditional PCA-based monitoring. In this approach, the selected
number of components to retain in the model is based on the cumulative
percentage of variance (CPV) equal to $90$. Control charts are constructed by using
the Hotelling's $T^{2}$ statistic and the $Q$ statistic (\cite{de2015overview}).

To detect an out-of-control condition, the control limits are set
such that the ARL for in-control observations is equal to 200 (this
corresponds to a significance level of 0.005). Each control limit
is calculated through 1,000 replications. The results are shown in
Figures~\ref{fig:Comparison-of-different-1}-\ref{fig:Comparison-of-different-3}. 

For $p=100$, as shown in Figure \ref{fig:Comparison-of-different-1}(a),
the proposed APC markedly outperforms the other benchmark methods. Even for shifts as small as $\delta=0.1\sigma$, the ARL for APC is $4.06$. This is about fourteen times smaller than the second best method, which is TRAS with ARL equal to $56$. Moreover, for shifts $\delta\geq0.1\sigma$, APC detects the shift almost instantly (i.e., $ARL_{1}=1$). 

The results for Scenario~II in Figure \ref{fig:Comparison-of-different-1}(b) also show that APC is superior to others, especially for moderate and large shifts. As an example, for a shift with the magnitude of
$\delta=0.25\sigma$, APC's ARL is $17.67$, while this values for the
best benchmark (TRAS) is $61.37$.  As expected, the out-of-control ARL values for all methods in Scenario~II is larger than those in Scenario~I. 

For higher dimensions, the APC's ability to detect shifts becomes even better while the other method's performances stay the same or deteriorate. This shows that as dimension grows, the shifts are easier to be captured in PC scores that are in the direction of the shift.
%Note that FDR is as good as APC for shifts greater than 0.5, but for smaller shifts, FDR has significantly larger ARL than APC. Moreover,

To summarize, this study indicates that the APC method outperforms other methods for detecting
small values of shifts. Also, as dimension grows APC works better in detecting a change promptly. This can be attributed to the adaptive nature of the proposed monitoring statistic. 

% Result of last Scenario is presented in Fig.~\ref{fig:Comparison-of-different-1}
% (c). As this figure shows, $T_new$ method outperforms others for shifts smaller than 1. However, it is followed by the APC method, and for shifts
% greater than 1, APC outperforms $T_{new}$ method. This result is expected
% because $T_{new}$ method is based on original (non-transformed) observation. And as it was discussed in scenario III, the original variables are all
% shifted and affected by this change. However, since we selected the shift to be exactly in the direction of top K (8 here) PCs, the PCs affected by this change are sparse. Therefore, $T_{new}$ method
% is capable of detecting smaller shifts faster because in the original domain this shift is represented in all variables and it may not be
% small. However, among PC based methods APC again is performing the best.

%\usepackage{graphicx}
%\usepackage{subcaption}

\begin{figure}
\centering
\subfloat[scenario I]{\includegraphics[width=0.47\linewidth]{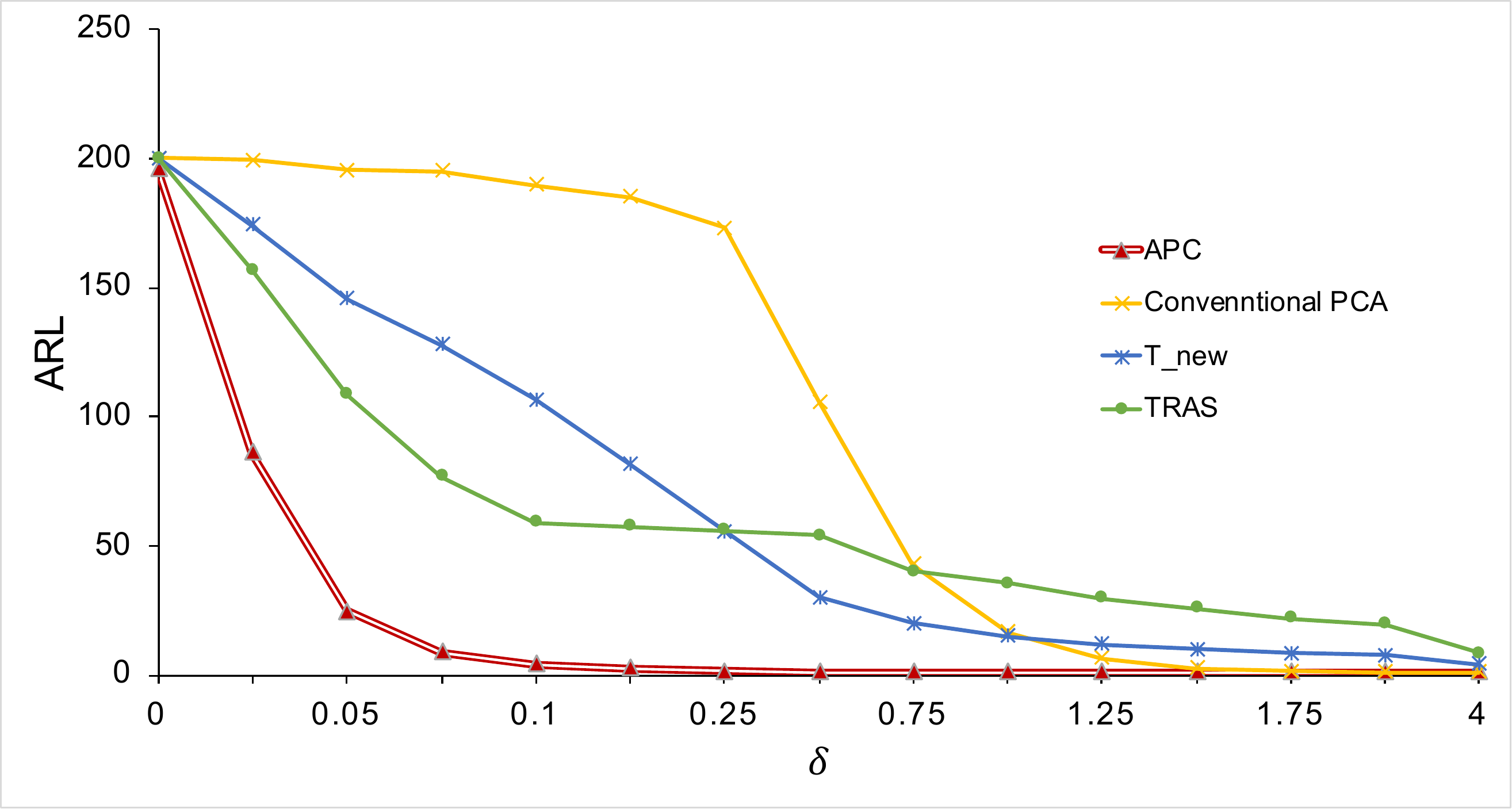}}
\subfloat[scenario II]{\includegraphics[width=0.47\linewidth]{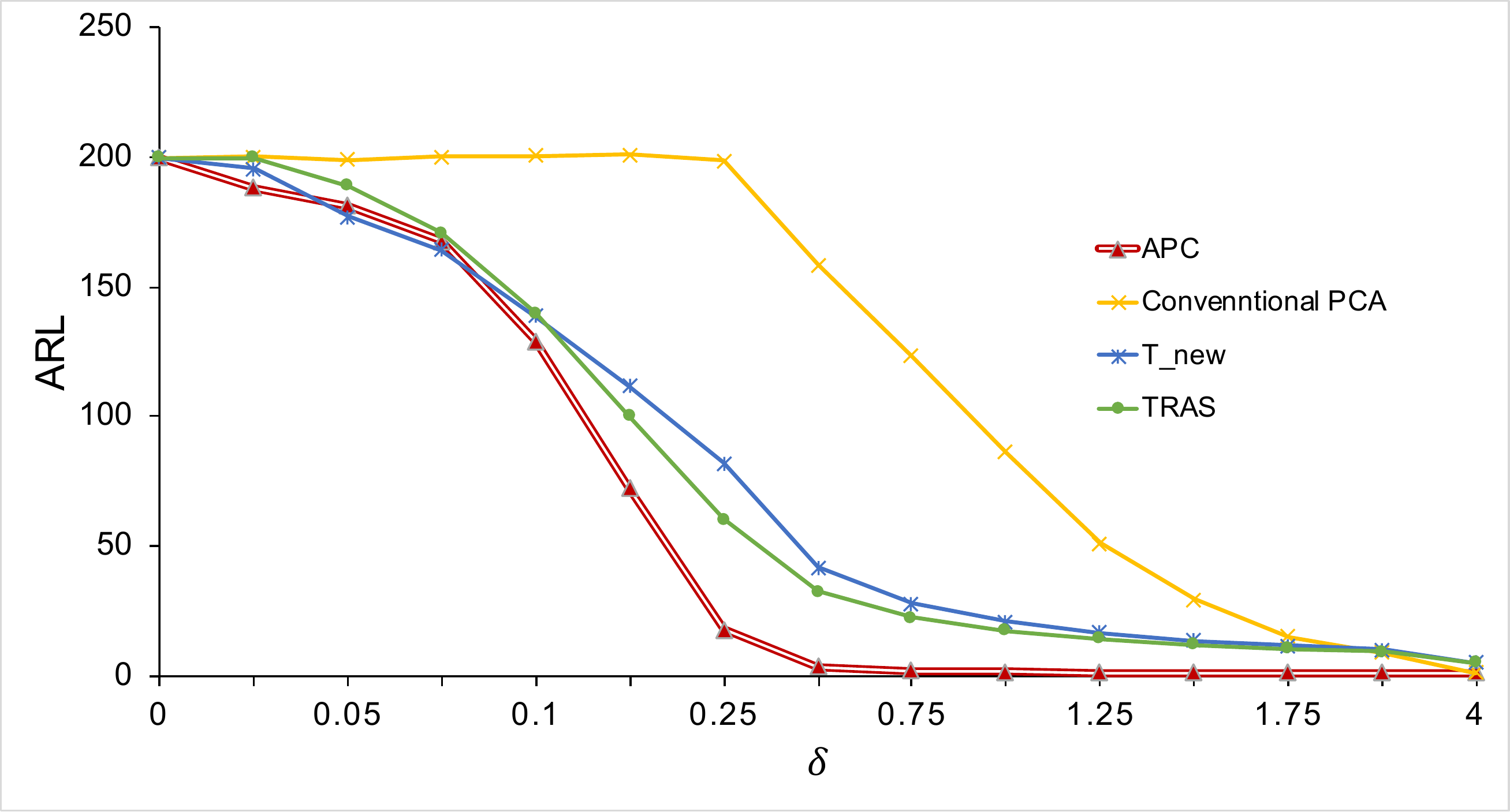}}
\caption{ARL of scenarios I, II for different values of $\delta$ (shift magnitude) for $p=100$}
\label{fig:Comparison-of-different-1} 
\end{figure}

\begin{figure}
\centering
\subfloat[scenario I]{\includegraphics[width=0.47\linewidth]{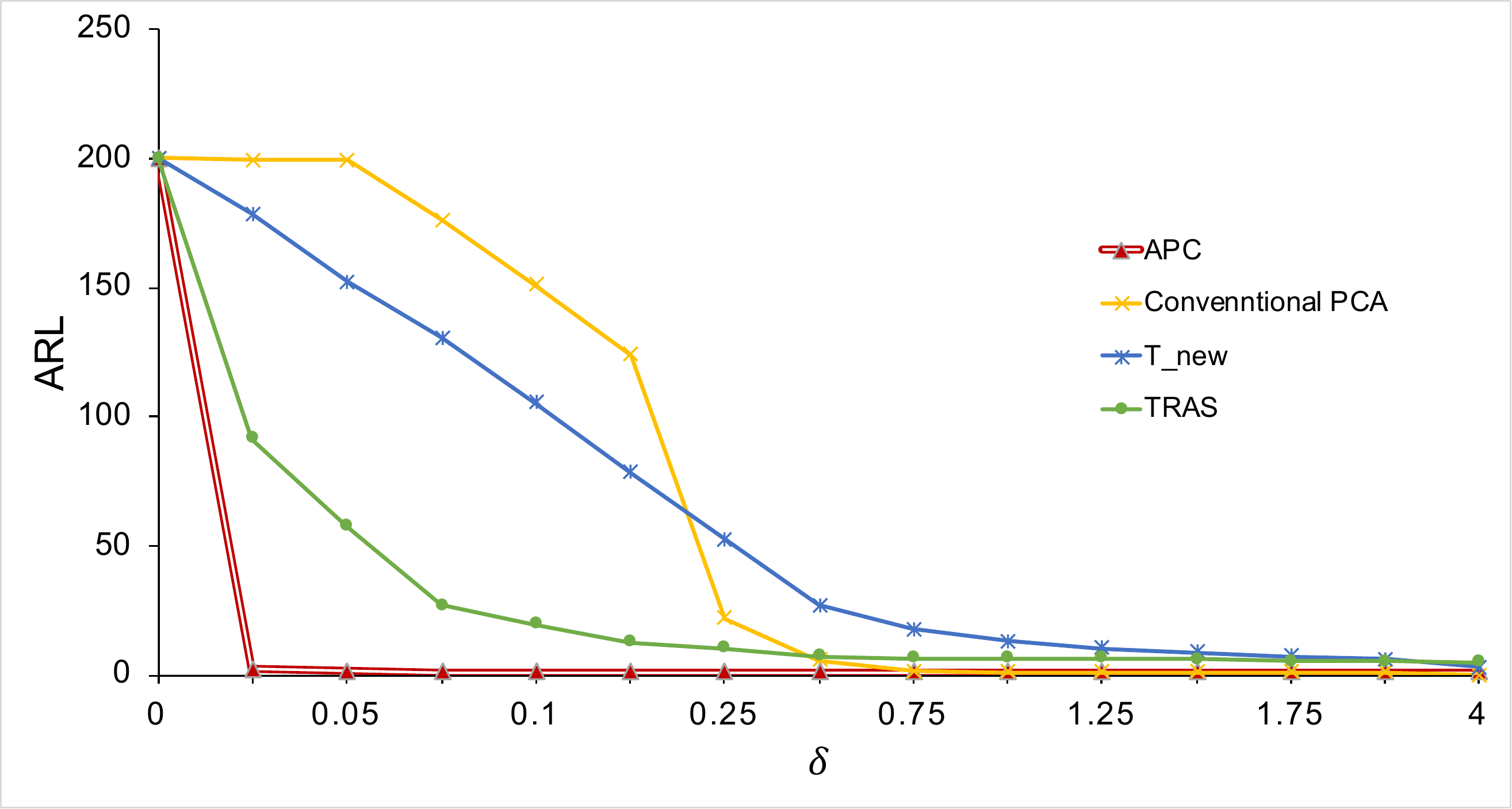}}
\subfloat[scenario II]{\includegraphics[width=0.47\linewidth]{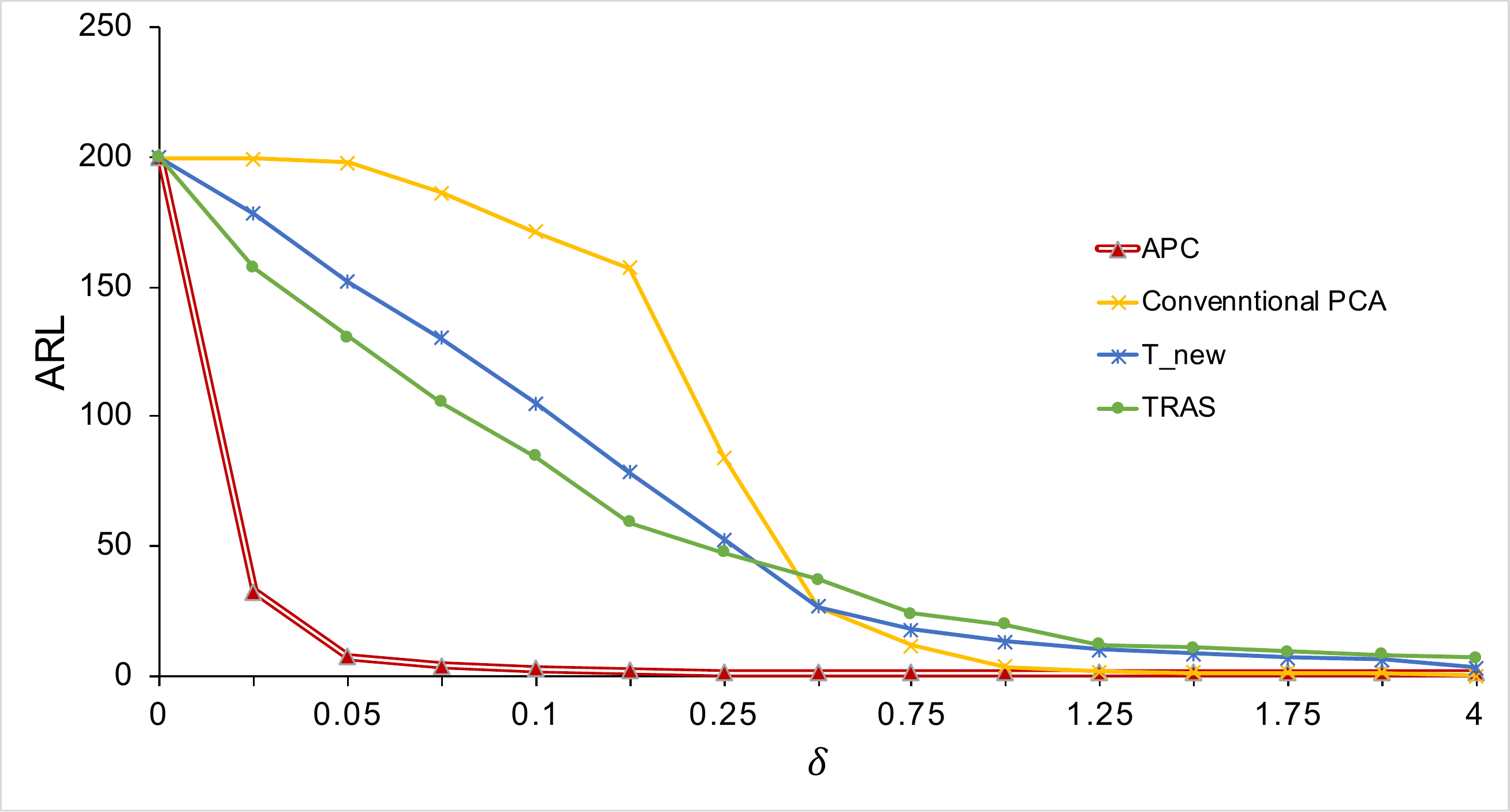}}
% \subfloat[snecario III]{\includegraphics[width=0.8\linewidth]{S3_n}}
\caption{ARL of scenarios I, II for different values of $\delta$ (shift magnitude) for $p=1000$}
\label{fig:Comparison-of-different-2} 
\end{figure}

\begin{figure}
\centering
\subfloat[scenario I]{\includegraphics[width=0.47\linewidth]{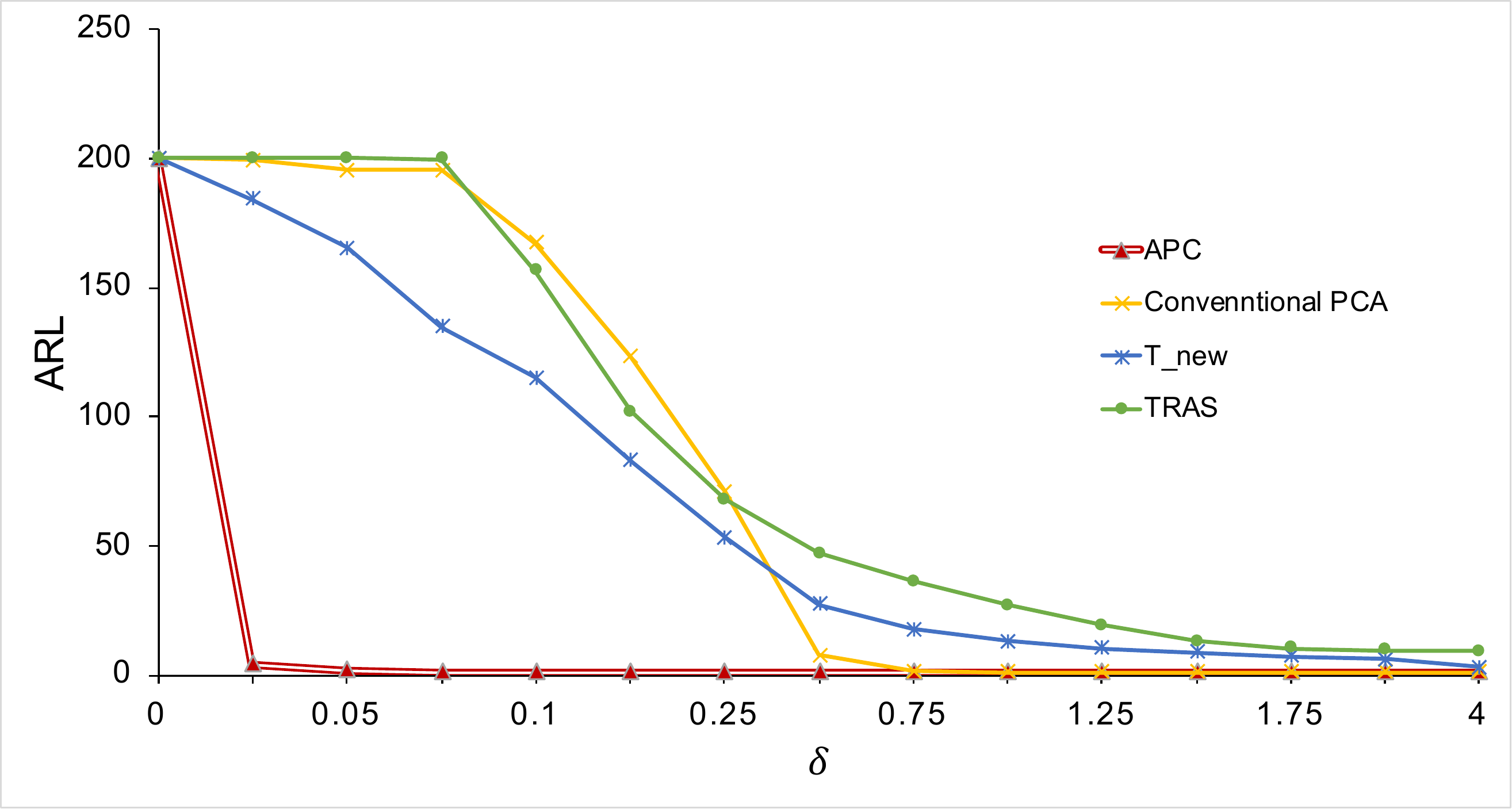}}
\subfloat[scenario II]{\includegraphics[width=0.47\linewidth]{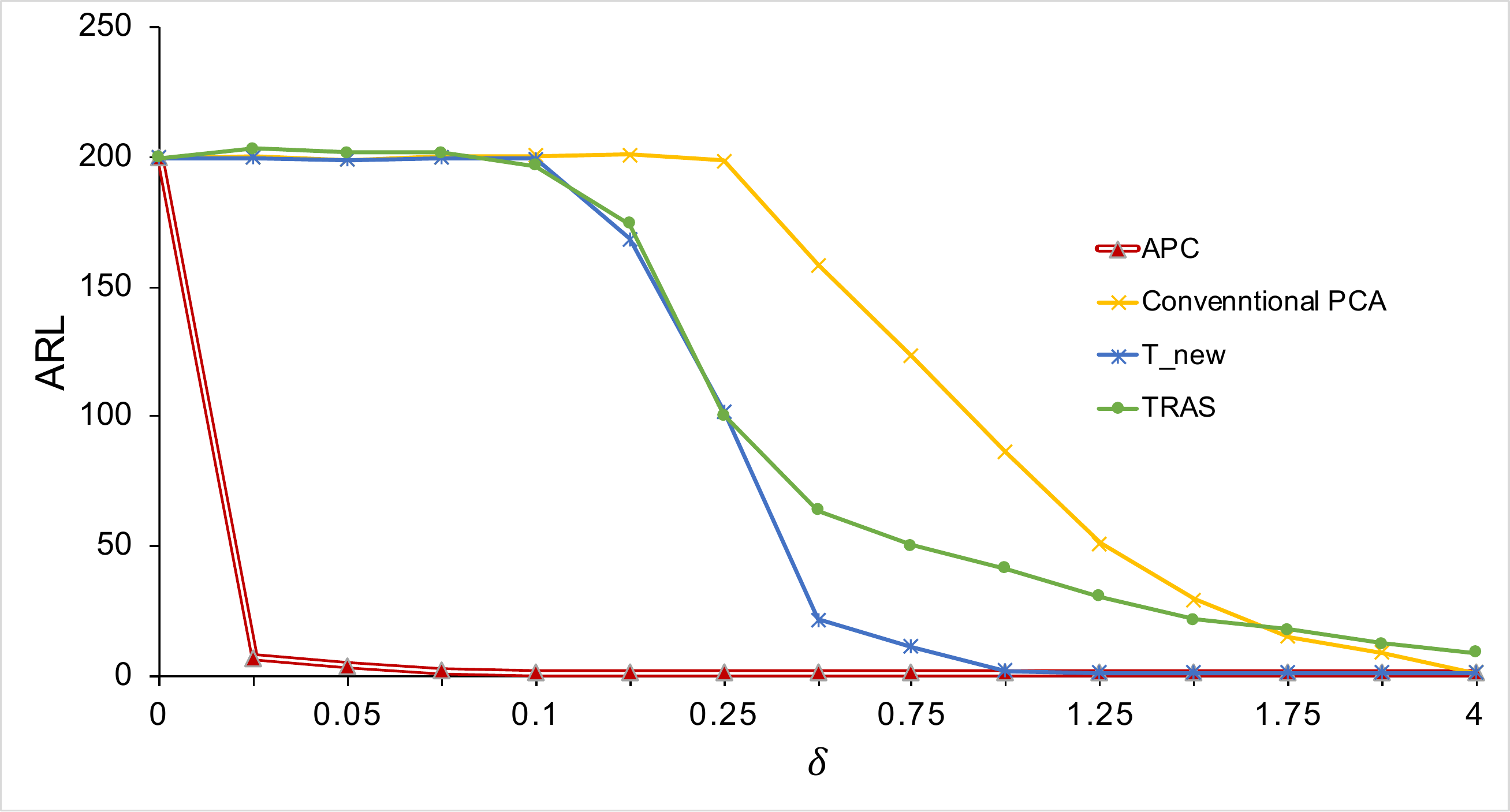}}
% \subfloat[snecario III]{\includegraphics[width=0.8\linewidth]{S3_n}}
\caption{ARL of scenarios I, II for different values of $\delta$ (shift magnitude) for $p=10,000$}
\label{fig:Comparison-of-different-3} 
\end{figure}

\subsection{Diagnosis Analysis}

In this section, in addition to Scenarios~I and II presented in the previous section, we add another scenario (Scenario~III) with an autoregressive covariance
matrix, i.e., $\rho_{ij}={|0.5|^{(i-j)}}$ for variables ${i,j}$. This covariance matrix impose the variables close to each other to have higher covariance and as the variables go farther in the matrix, their covariance becomes smaller.
We validate the performance of the proposed diagnosis method for different
percentages of shifted variables (PS) as well as their shift magnitudes
($\delta$), using the following performance measures: (a) false negative
percentage (\%FN), defined as the percentage of the number of variables
that are not detected over the number of all faulty variables; (b)
false positive percentage (\%FP), defined as the percentage of the ratio of number of variables that are mistakenly detected as faulty
over the number of all not-faulty variables; (c) parameter selection
score (PSS), defined as the total number of variables that are labeled
incorrectly (either as faulty or not-faulty); and (d) F1-Score, defined
as the harmonic average of the precision and recall, and indicates
our overall performance combining the FP and FN. 

For FN, FP and PSS measures, the smaller the value, the better
the performance, whereas for F1-score the higher the better. We compare
the performance of our proposed method with the Lasso-based  diagnosis approach proposed by \cite{zou2011lasso} called LEB. LEB
is an LASSO-based diagnostic approach approach for
diagnosis of sparse changes using BIC and the adaptive LASSO variable selection. The comparison results for $p=100$  are shown in Tables~\ref{tab:Diag_compare_sen1}-\ref{tab:Diag_compare_sen4}, 
and Figures \ref{fig:F1_sen1_compare}-\ref{fig:F1_sen4_compare} for scenarios I, II, and III, respectively.  

In these tables PS denotes the percentage
of shifted process variables. As shown in Table~\ref{tab:Diag_compare_sen1}
and Figure~\ref{fig:F1_sen1_compare}, under scenario I, when shift
occurs in a random set of variables with a random covariance matrix,
our proposed PCSR performs better than the LEB (\cite{zou2011lasso}) for most of the cases
except for the case with PS=10\% and small shifts (i.e., $\delta=0.7\sigma$
). Even in this case, PCST is close to LEB. However,
for larger shifts, PCSR outperforms LEB. For instance, for a shift equal to $1\sigma$ and percentage of shifted variables equal to $10\%$ the $F1$ accuracy using PCSR is equal to $0.9881$ while it is equal to $0.9363$ for LEB. 

In Scenario II, PCSR
clearly outperforms the LEB method. For example for a shift equal to $0.7\sigma$
on 10\% of variables, PCSR's F1-score is 0.6802 while, the LEB F1 score is 0.3725 (see Table~\ref{tab:Diag_compare_sen2} and Figure~\ref{fig:F1_sen2_compare}). 

Also, the results in scenario III indicate the superior performance of that our method (see Table~\ref{tab:Diag_compare_sen4} and Figure~\ref{fig:F1_sen4_compare}). For instance, for a  $0.5\sigma$ shift on 25\% of variables, PCSR's F1-score is 0.7173 and 0.4648 for LEB. 

These results show that for random non-sparse covariance matrices, the LEB method and PCSR method performs almost similarly. However, for sparse covariance matrices such as a block covariance or an autoregressive covariance, PCSR clearly outperforms LEB method. These sparse occurance of covariance matrices are very common in real world. This is because of the fact that in many situations, each data stream is correlated only with a small group of other data streams, but is not correlated with all other data streams collected in the system. Hence, a method that can detect the changes in such systems is necessary and more appropriate for real-world applications.

In sum, the results of the simulation study show the effectiveness
of our method in identifying the set of altered variables and its superiority over the current state-of-the-art.

%%%%%%%%%%%%%%%%%%%%%%%%%%%%%%%%%
%%%%%%%%% SCENARIO I %%%%%%%%%%%
%%%%%%%%%%%%%%%%%%%%%%%%%%%%%%%%%
\begin{table}[!htb]
\centering %
\begin{tabular}{cc|cccc|cccc}
\hline 
\multicolumn{1}{c}{} & \multicolumn{1}{c}{\textbf{Shift}} & \multicolumn{4}{c}{\textbf{PCSR}} & \multicolumn{4}{c}{\textbf{LEB}}\tabularnewline \hline 
\multirow{8}{*}{\rotatebox[origin=c]{90}{\textbf{PS=0.1}}}  &  & \textbf{FP\%}  & \textbf{FN\%}  & \textbf{PSS}  & \textbf{F1}  & \textbf{FP\%}  & \textbf{FN\%}  & \textbf{PSS}  & \textbf{F1} \tabularnewline
 & 0.1  & 86.99  & 5.289  & 13.459  & 0.1603  & 85.86  & 4.28  & 12.438  & \textbf{0.1790} \tabularnewline
% & 0.2 & 64.55 & 8.1511 & 13.791 & 0.3393 & 62.18 & 6.1067 & 11.714 & \textbf{0.3907} \\
 & 0.3  & 41.36  & 8.652  & 11.923  & 0.5001  & 38.11  & 6.757  & 9.891  & \textbf{0.5554} \tabularnewline
% & 0.4 & 22.04 & 7.8867 & 9.302 & 0.6352 & 20.28 & 6.7222 & 8.078 & \textbf{0.6679} \\
 & 0.5  & 8.05  & 6.107  & 6.301  & 0.757  & 6.56  & 5.501  & 5.607  & \textbf{0.7763} \tabularnewline
% & 0.6 & 2.6 & 3.9011 & 3.771 & 0.8501 & 1.45 & 3.7667 & 3.535 & \textbf{0.8537} \\
 & 0.7  & 0.62  & 2.272  & 2.107  & \textbf{0.9141}  & 0.62  & 3.924  & 3.594  & 0.8538 \tabularnewline
% & 0.8 & 0.08 & 1.0056 & 0.913 & \textbf{0.9613} & 0.07 & 2.5589 & 2.31 & 0.9010 \\
% & 0.9 & 0.01 & 0.49889 & 0.45 & \textbf{0.9801} & 0.01 & 1.9378 & 1.745 & 0.9232 \\
 & 1  & 0  & 0.29  & 0.261  & \textbf{0.9881}  & 0  & 1.580  & 1.422  & 0.9363 \tabularnewline
 & 1.25  & 0  & 0.1567  & 0.141  & \textbf{0.9934}  & 0  & 0.761  & 0.685  & 0.968 \tabularnewline
 & 1.5  & 0  & 0.156  & 0.14  & \textbf{0.9935}  & 0  & 0.347  & 0.312  & 0.9850 \tabularnewline
\hline 
% & 2 & 0 & 0.16667 & 0.15 & \textbf{0.99304} & 0 & 0.19556 & 0.176 & 0.99168 \\
% & 3 & 0 & 0.15667 & 0.141 & \textbf{0.9934} & 0 & 0.066667 & 0.06 & 0.99716 \\ 
\multirow{7}{*}{\rotatebox[origin=c]{90}{\textbf{PS=0.15}}}  & 0.1  & 87.267  & 6.388  & 18.52  & \textbf{0.1672}  & 92.067  & 3.466  & 16.756  & 0.1158 \tabularnewline
% & 0.2 & 64.78 & 9.6859 & 17.95 & 0.3686 & 65.013 & 8.3071 & 16.813 & \textbf{0.3773} \\
 & 0.3  & 42.773  & 10.231  & 15.112  & \textbf{0.5323}  & 45.84  & 9.379  & 14.848  & 0.520 \tabularnewline
% & 0.4 & 22.973 & 9.0729 & 11.158 & 0.6782 & 22.647 & 8.3718 & 10.513 & \textbf{0.6884} \\
 & 0.5  & 9.16  & 7.029  & 7.349  & \textbf{0.7927}  & 8.853  & 7.077  & 7.343  & 0.7916 \tabularnewline
% & 0.6 & 2.7067 & 4.7929 & 4.48 & \textbf{0.8734} & 2.4133 & 5.5153 & 5.05 & 0.8573 \\
 & 0.7  & 0.687  & 2.68  & 2.381  & \textbf{0.9310}  & 0.433  & 3.914  & 3.392  & 0.9015 \tabularnewline
% & 0.8 & 0.10667 & 1.3718 & 1.182 & \textbf{0.9647} & 0.13333 & 3.0106 & 2.579 & 0.9235 \\
% & 0.9 & 0 & 0.58941 & 0.501 & \textbf{0.9844} & 0 & 1.7271 & 1.468 & 0.9550 \\
 & 1  & 0  & 0.364 & 0.309  & \textbf{0.9903}  & 0  & 1.711  & 1.454  & 0.9552 \tabularnewline
 & 1.25  & 0  & 0.242  & 0.206  & \textbf{0.9935}  & 0  & 0.854  & 0.726  & 0.9772 \tabularnewline
 & 1.5  & 0  & 0.229  & 0.195  & \textbf{0.9938}  & 0  & 0.439  & 0.373  & 0.9882 \tabularnewline
\hline 
%& 2 & 0 & 0.23294 & 0.198 & 0.99377 & 0 & 0.22118 & 0.188 & \textbf{0.99399} \\
% & 3 & 0 & 0.22471 & 0.191 & 0.99399 & 0 & 0.098824 & 0.084 & \textbf{0.9973} \\ 
\multirow{7}{*}{\rotatebox[origin=c]{90}{\textbf{PS=0.25}}}  & 0.1  & 86.676  & 7.303  & 27.146  & \textbf{0.1941}  & 88.632  & 4.815  & 25.769  & 0.1769 \tabularnewline
%& 0.2 & 66.332 & 10.872 & 24.737 & \textbf{0.4016} & 71.128 & 7.7133 & 23.567 & 0.3747 \\
 & 0.3  & 48.268  & 10.277  & 19.775  & \textbf{0.5653}  & 52.116  & 8.736  & 19.581  & 0.5460 \tabularnewline
% & 0.4 & 30.444 & 8.8293 & 14.233 & \textbf{0.7089} & 33.256 & 8.4013 & 14.615 & 0.6934 \\
 & 0.5  & 15.896  & 6.939  & 9.178  & \textbf{0.8208}  & 19.848  & 6.372  & 9.741  & 0.8028 \tabularnewline
% & 0.6 & 5.68 & 5.4853 & 5.534 & \textbf{0.8961} & 8.772 & 6.348 & 6.954 & 0.8673 \\
 & 0.7  & 1.228  & 3.123  & 2.649  & \textbf{0.9505}  & 3.004  & 6.123  & 5.343  & 0.9018 \tabularnewline
% & 0.8 & 0.208 & 1.592 & 1.246 & \textbf{0.9765} & 0.388 & 4.588 & 3.538 & 0.9351 \\
% & 0.9 & 0.028 & 0.816 & 0.619 & \textbf{0.9882} & 0.028 & 3.796 & 2.854 & 0.9472 \\
 & 1  & 0  & 0.489  & 0.367  & \textbf{0.9929}  & 0  & 3.145  & 2.359  & 0.956 \tabularnewline
 & 1.25  & 0  & 0.371  & 0.278  & \textbf{0.9946}  & 0  & 1.876  & 1.407  & 0.9732 \tabularnewline
 & 1.5  & 0  & 0.368  & 0.276  & \textbf{0.9947}  & 0  & 1.161  & 0.871  & 0.9832 \tabularnewline
% & 2 & 0 & 0.35467 & 0.266 & \textbf{0.99485} & 0 & 0.588 & 0.441 & 0.99143 \\
% & 3 & 0 & 0.344 & 0.258 & 0.995 & 0 & 0.22533 & 0.169 & \textbf{0.9967}
\end{tabular}\caption{Diagnosis simulation results for Scenario I \label{tab:Diag_compare_sen1} }
\end{table}

\begin{figure}[!htb]
\centering \subfloat[PS=0.10]{\includegraphics[width=0.45\textwidth]{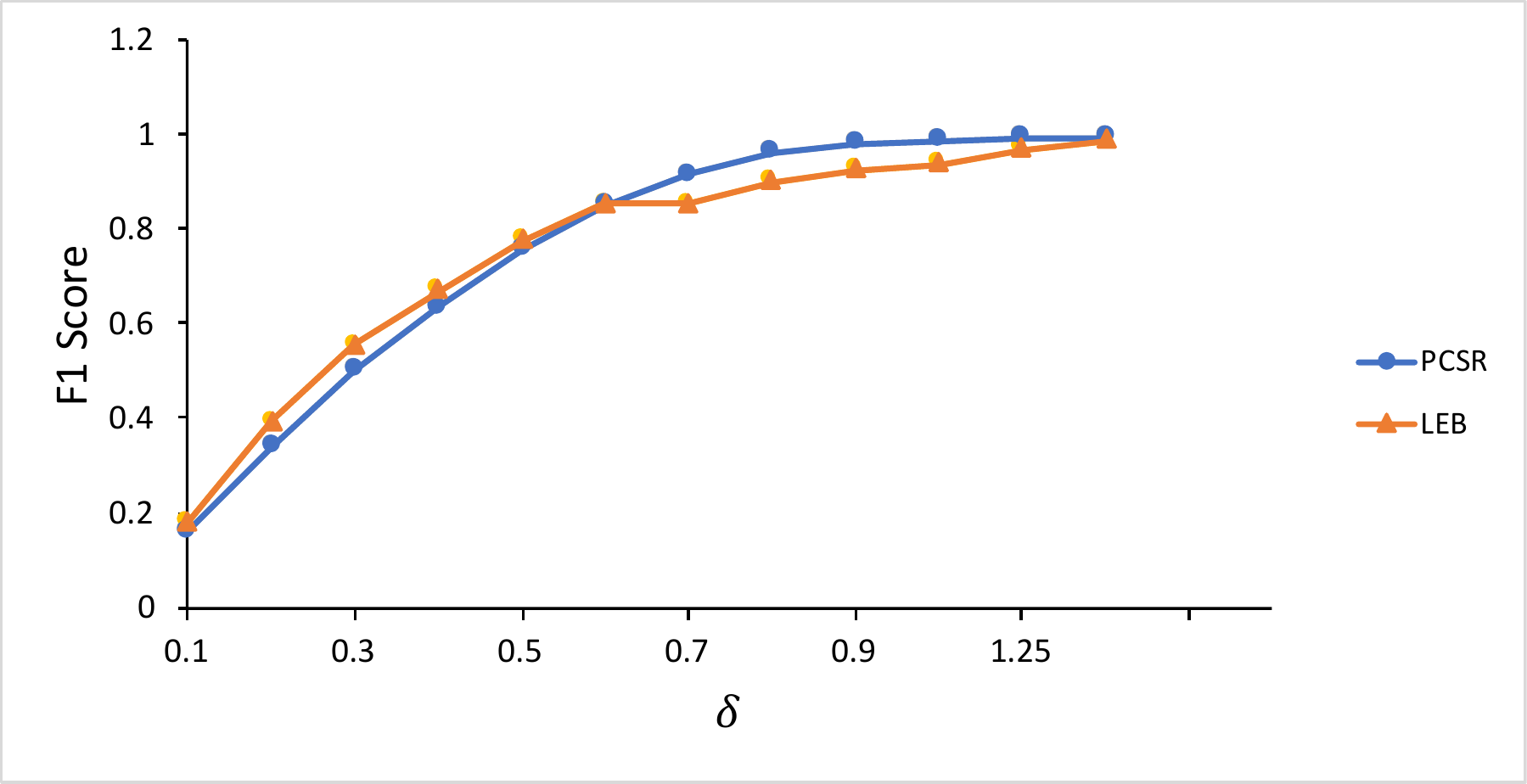}}
\vspace{0.05\linewidth}
\subfloat[PS=0.15]{\includegraphics[width=0.45\textwidth]{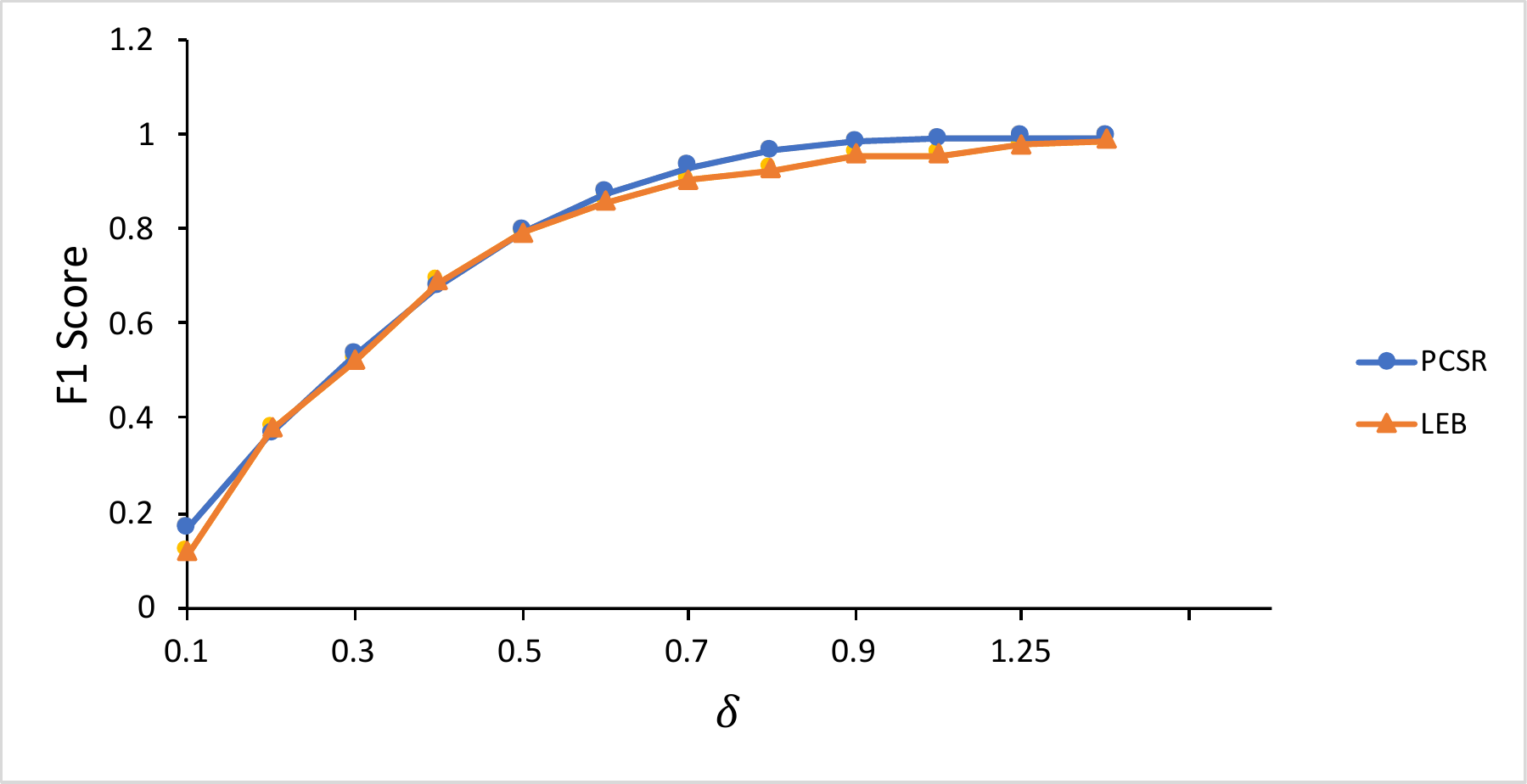}}
\hspace{0.5mm} \subfloat[PS=0.25]{\includegraphics[width=0.45\linewidth]{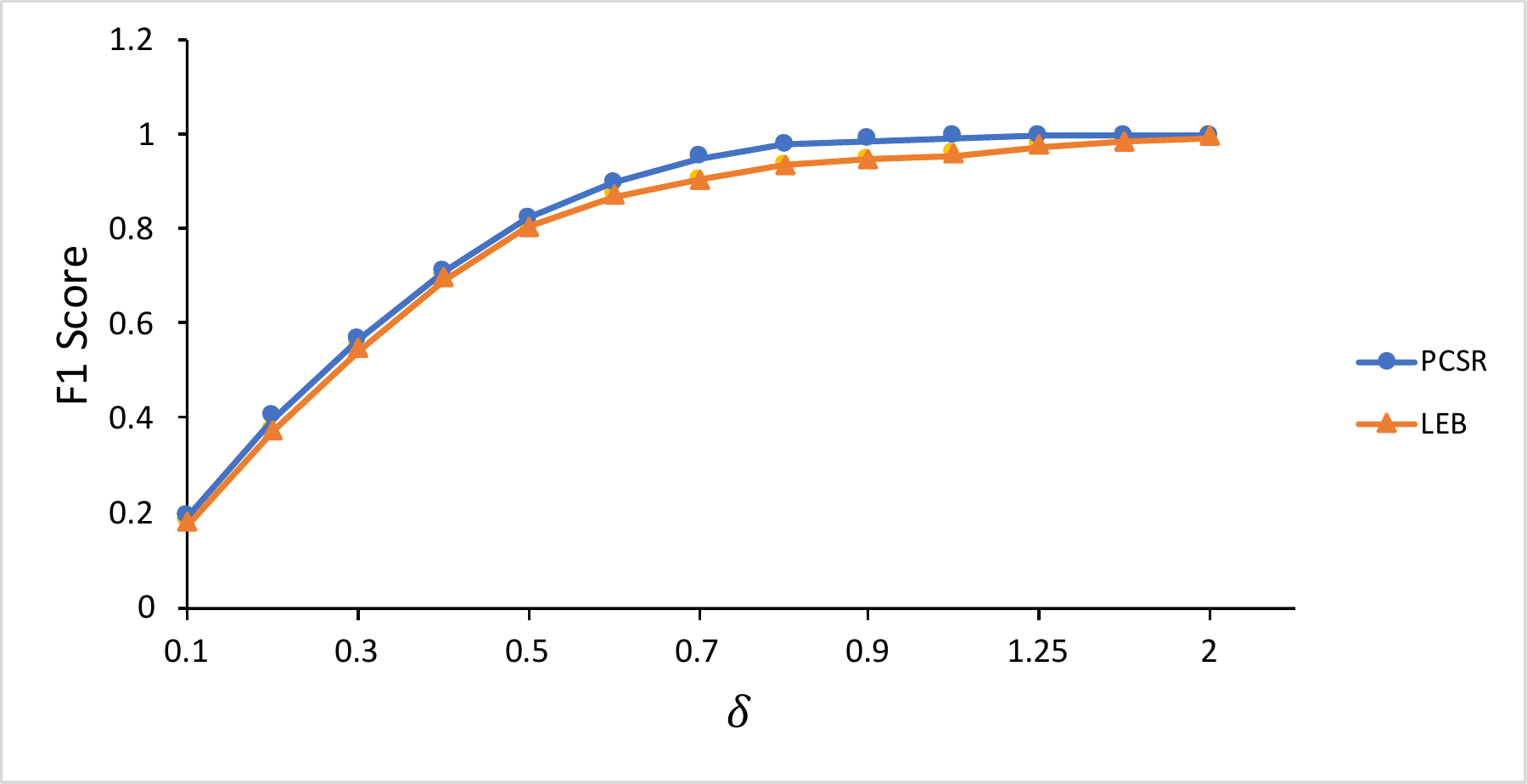}}
\caption{F1 of scenarios I different values of $\delta$ (shift magnitude)
}
\label{fig:F1_sen1_compare} 
\end{figure}

%%%%%%%%%%%%%%%%%%%%%%%%%%%%%%%%%
%%%%%%%%% SCENARIO II %%%%%%%%%%%
%%%%%%%%%%%%%%%%%%%%%%%%%%%%%%%%%

\begin{table}[!htb]
\centering %
\begin{tabular}{cc|cccc|cccc}
\hline 
\multicolumn{1}{c}{\textbf{PS}} & \multicolumn{1}{c}{\textbf{Shift}} & \multicolumn{4}{c}{\textbf{PCSR}} & \multicolumn{4}{c}{\textbf{LEB}}\tabularnewline
\hline 
\multicolumn{1}{c}{} & \multicolumn{1}{c|}{} & \multicolumn{1}{c}{\textbf{\%FP}} & \multicolumn{1}{c}{\textbf{\%FN}} & \multicolumn{1}{c}{\textbf{PSS}} & \multicolumn{1}{c|}{\textbf{F1}} & \multicolumn{1}{c}{\textbf{\%FP}} & \multicolumn{1}{c}{\textbf{\%FN}} & \multicolumn{1}{c}{\textbf{PSS}} & \multicolumn{1}{c}{\textbf{F1}}\tabularnewline
\multirow{7}{*}{\rotatebox[origin=c]{90}{\textbf{PS=0.10}}}  & 0.1  & 93.775  & 0.40435  & 7.874  & 0.1011  & 93.663  & 0.91304  & 8.333  & \textbf{0.1049} \tabularnewline
% & 0.2 & 82.8 & 0.84891 & 7.405 & \textbf{0.2646} & 86.387 & 0.94783 & 7.783 & 0.2172 \\
 & 0.3  & 78.412  & 0.5337  & 6.764  & \textbf{0.3257}  & 81.562  & 0.93478  & 7.385  & 0.2784 \tabularnewline
% & 0.4 & 72.4 & 0.57065 & 6.317 & \textbf{0.3939} & 80.725 & 0.925 & 7.309 & 0.2877 \\
 & 0.5  & 63.1  & 0.75652  & 5.744  & \textbf{0.4854}  & 76.987  & 0.81848  & 6.912  & 0.3333 \tabularnewline
% & 0.6 & 51.45 & 0.94891 & 4.989 & \textbf{0.5927} & 77.438 & 0.73804 & 6.874 & 0.3277 \\
 & 0.7  & 40.387  & 1.0446  & 4.192  & \textbf{0.6802}  & 72.162  & 0.96848  & 6.664  & 0.3725 \tabularnewline
% & 0.8 & 29.288 & 1.1315 & 3.384 & \textbf{0.7556} & 65.963 & 1.1576 & 6.342 & 0.4314 \\
% & 0.9 & 18.95 & 1.1011 & 2.529 & \textbf{0.8273} & 48.325 & 1.8967 & 5.611 & 0.5715 \\
 & 1  & 8.575  & 1.2109  & 1.8  & \textbf{0.8875}  & 33.562  & 2.4522  & 4.941  & 0.6672 \tabularnewline
 & 1.25  & 0.7375  & 0.71413  & 0.716  & \textbf{0.9586}  & 12.425  & 3.5022  & 4.216  & 0.7664 \tabularnewline
 & 1.5  & 0  & 0.40217  & 0.37  & \textbf{0.9789}  & 0.6625  & 3.612  & 3.376  & 0.8302 \tabularnewline
\hline 
%& 2 & 0 & 0.125 & 0.115 & \textbf{0.99346} & 0 & 2.9054 & 2.673 & 0.86145 \\
%& 3 & 0 & 0.081522 & 0.075 & \textbf{0.99572} & 0 & 2.1793 & 2.005 & 0.89255 \\ 
\multirow{7}{*}{\rotatebox[origin=c]{90}{\textbf{PS=0.15}}}  & 0.1  & 93.537  & 0.50714  & 15.392  & 0.1119  & 93.419  & 0.75476  & 15.581  & \textbf{0.1160} \tabularnewline
% & 0.2 & 83.85 & 0.94643 & 14.211 & \textbf{0.2631} & 86.494 & 1.3512 & 14.974 & 0.2229 \\
 & 0.3  & 79.588  & 0.69405  & 13.317  & \textbf{0.3225}  & 84.9  & 1.1179  & 14.523  & 0.2475 \tabularnewline
% & 0.4 & 74.244 & 0.74405 & 12.504 & \textbf{0.3867} & 82.219 & 1.1524 & 14.123 & 0.2820 \\
 & 0.5  & 64.444  & 0.98333  & 11.137  & \textbf{0.4925}  & 77.825  & 1.3536  & 13.589  & 0.3355 \tabularnewline
% & 0.6 & 53.9 & 1.2036 & 9.635 & \textbf{0.5914} & 76.237 & 1.3238 & 13.31 & 0.3521 \\
 & 0.7  & 42.219  & 1.3167  & 7.861  & \textbf{0.6891}  & 68.863  & 1.7393  & 12.479  & 0.4258 \tabularnewline
% & 0.8 & 28.363 & 1.4405 & 5.748 & \textbf{0.7882} & 56.163 & 2.15 & 10.792 & 0.5501 \\
% & 0.9 & 16.575 & 1.4964 & 3.909 & \textbf{0.8644} & 48.375 & 2.5476 & 9.88 & 0.6119 \\
 & 1  & 7.8375  & 1.519  & 2.53  & \textbf{0.9182}  & 39.062  & 2.9917  & 8.763  & 0.6712 \tabularnewline
 & 1.25  & 0.81875  & 1.0476  & 1.011  & \textbf{0.9696}  & 7.0625  & 5.8714  & 6.062  & 0.8282 \tabularnewline
 & 1.5  & 0.00625  & 0.52262  & 0.44  & \textbf{0.9869}  & 0.225  & 5.969  & 5.05  & 0.8662 \tabularnewline
\hline 
%& 2 & 0 & 0.21667 & 0.182 & \textbf{0.99461} & 0 & 4.5393 & 3.813 & 0.89573 \\
% & 3 & 0 & 0.10119 & 0.085 & \textbf{0.99746} & 0 & 2.9655 & 2.491 & 0.92921 \\ \cline{2-10} 
\multirow{7}{*}{\rotatebox[origin=c]{90}{\textbf{PS=0.25}}}  & 0.1  & 95.088  & 0.56447  & 23.25  & \textbf{0.0878}  & 96.525  & 0.87105  & 23.828  & 0.0641 \tabularnewline
% & 0.2 & 85.321 & 1.1684 & 21.365 & \textbf{0.2440} & 87.908 & 1.5092 & 22.245 & 0.2043 \\
 & 0.3  & 79.483  & 1.0329  & 19.861  & \textbf{0.3251}  & 84.133  & 1.5408  & 21.363  & 0.2593 \tabularnewline
% & 0.4 & 71.3 & 1.2684 & 18.076 & \textbf{0.4219} & 82.392 & 1.45 & 20.876 & 0.2852 \\
 & 0.5  & 59.008  & 1.7592  & 15.499  & \textbf{0.5489}  & 80.083  & 1.3763  & 20.266  & 0.3140 \tabularnewline
% & 0.6 & 46.683 & 1.8842 & 12.636 & \textbf{0.6625} & 70.1 & 2.0526 & 18.384 & 0.4172 \\
 & 0.7  & 36.492  & 1.7855  & 10.115  & \textbf{0.7440}  & 64.933  & 2.2645  & 17.305  & 0.4742 \tabularnewline
% & 0.8 & 24.338 & 1.7316 & 7.157 & \textbf{0.8290} & 25.604 & 5.3145 & 10.184 & 0.7704 \\
% & 0.9 & 13.342 & 1.8066 & 4.575 & \textbf{0.8973} & 27.387 & 4.9211 & 10.313 & 0.7593 \\
 & 1  & 6.4167  & 1.8013  & 2.909  & \textbf{0.9379}  & 15.758  & 6.4645  & 8.695  & 0.8180 \tabularnewline
 & 1.25  & 0.57083  & 1.1645  & 1.022  & \textbf{0.9792}  & 2.625  & 7.0224  & 5.967  & 0.8878 \tabularnewline
 & 1.5  & 0.0083333  & 0.62895  & 0.48  & \textbf{0.9903}  & 0.1  & 7.1421  & 5.452  & 0.8993 \tabularnewline
%& 2 & 0 & 0.20921 & 0.159 & \textbf{0.99678} & 0 & 5.4553 & 4.146 & 0.92162 \\
%& 3 & 0 & 0.1 & 0.076 & \textbf{0.99847} & 0 & 3.9197 & 2.979 & 0.94228
\end{tabular}\caption{Diagnosis simulation results for Scenario II \label{tab:Diag_compare_sen2} }
\end{table}

\begin{figure}[!htb]
\centering \subfloat[PS=0.10]{\includegraphics[width=0.45\textwidth]{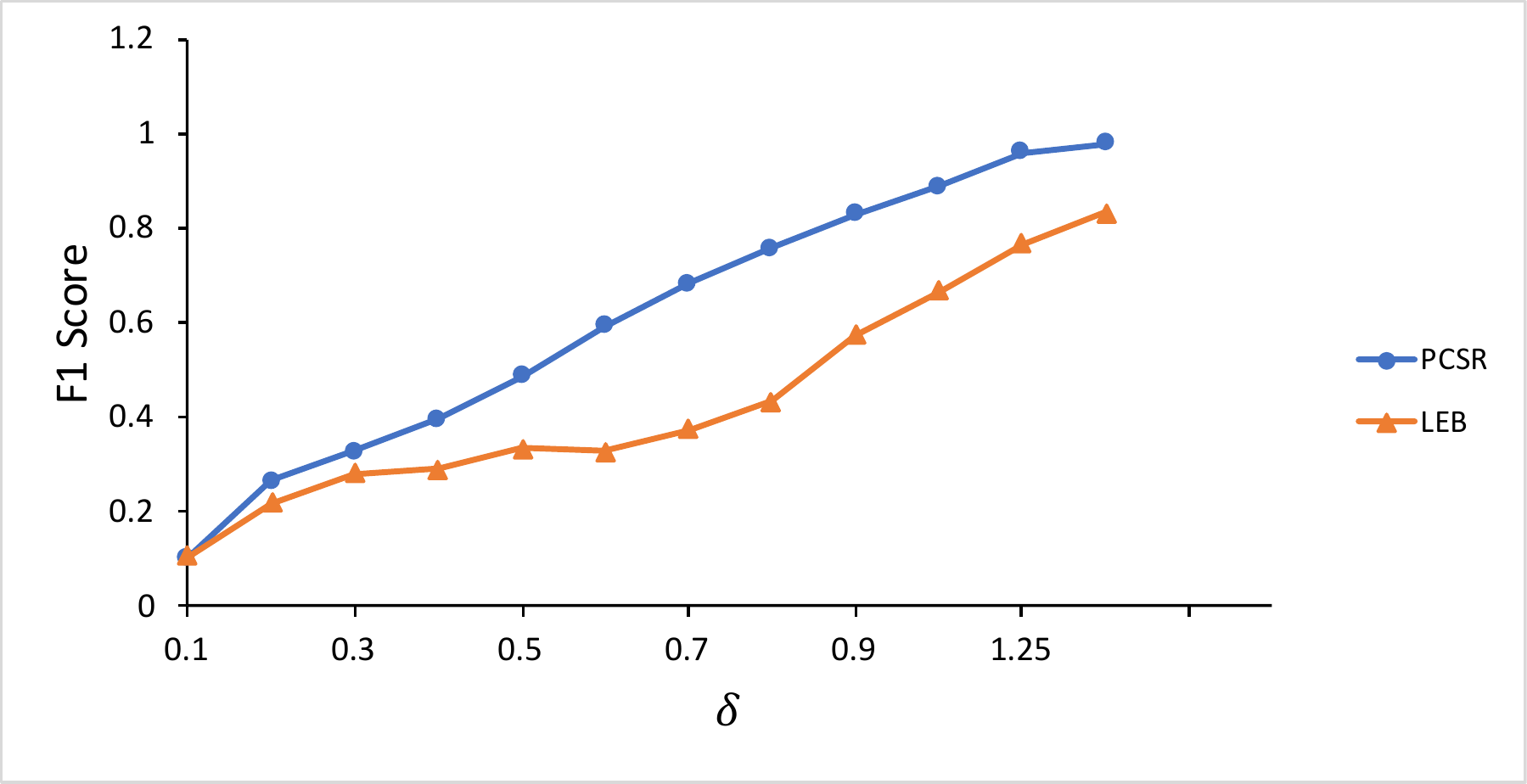}}
\vspace{0.05\linewidth}
\subfloat[PS=0.15]{\includegraphics[width=0.45\textwidth]{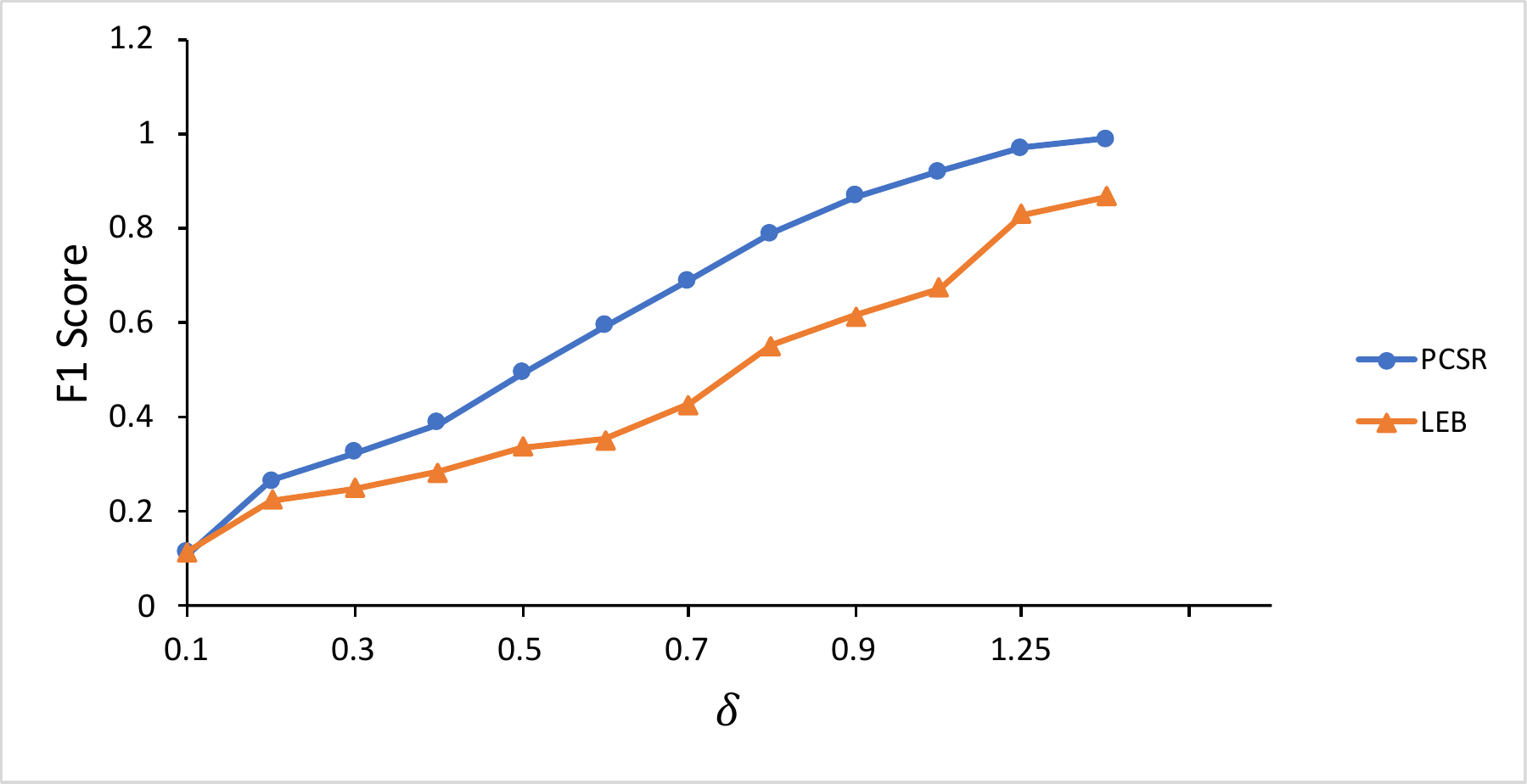}}
\hspace{0.5mm} \subfloat[PS=0.25]{\includegraphics[width=0.45\linewidth]{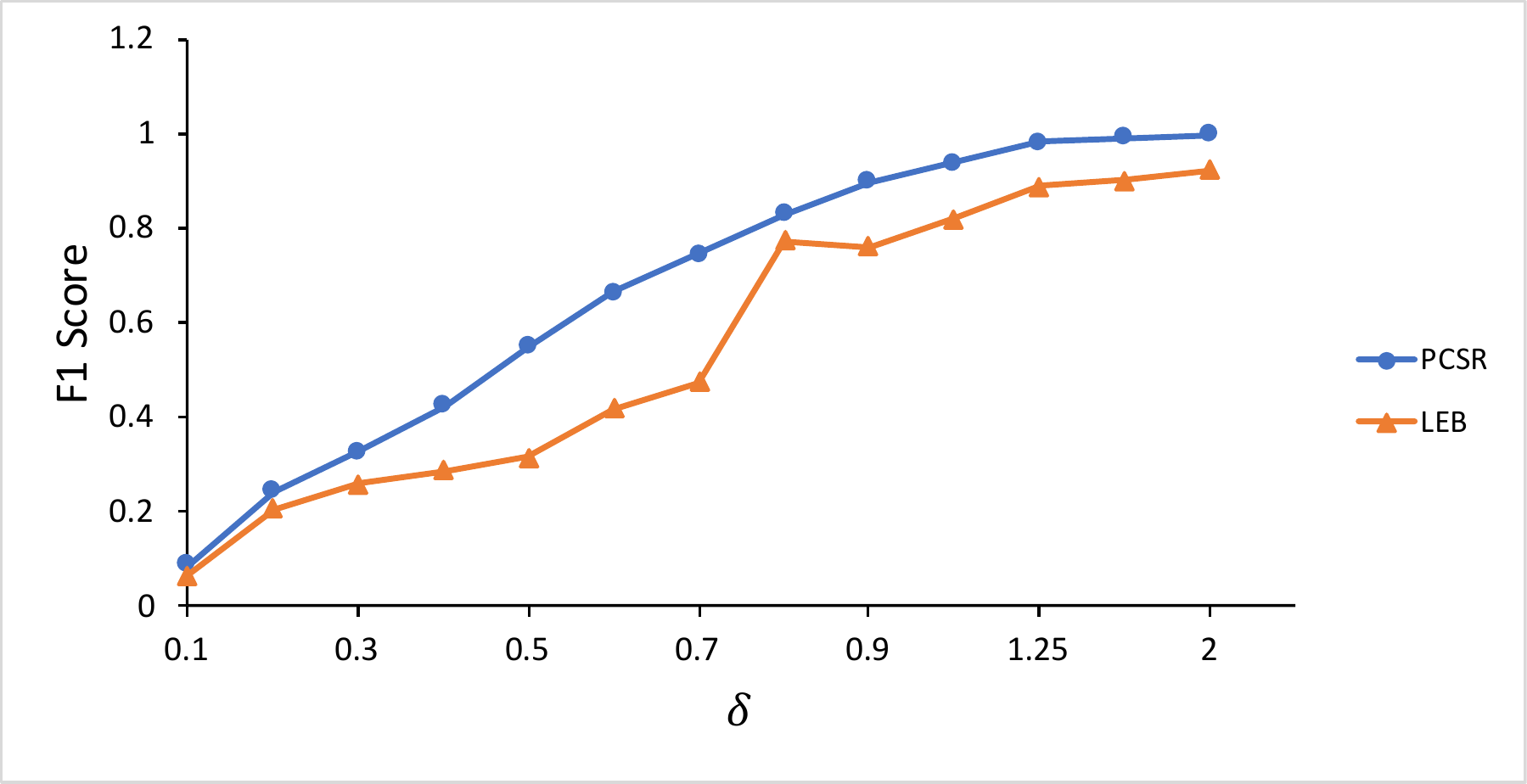}}
\caption{F1 of scenarios II different values of $\delta$ (shift magnitude)
}
\label{fig:F1_sen2_compare} 
\end{figure}

%%%%%%%%%%%%%%%%%%%%%%%%%%%%%%%%%
%%%%%%%%% SCENARIO VI %%%%%%%%%%%
%%%%%%%%%%%%%%%%%%%%%%%%%%%%%%%%%

\begin{table}[!htb]
\centering %
\begin{tabular}{cc|cccc|cccc}
\hline 
\multicolumn{1}{c}{} & \multicolumn{1}{c}{\textbf{Shift}} & \multicolumn{4}{c}{\textbf{PCSR}} & \multicolumn{4}{c}{\textbf{LEB}}\tabularnewline
\hline 
 &   & \textbf{\%FP}  & \textbf{\%FN}  & \textbf{PSS}  & \textbf{F1}  & \textbf{\%FP}  & \textbf{\%FN}  & \textbf{PSS}  & \textbf{F1} \tabularnewline
\multirow{7}{*}{\rotatebox[origin=c]{90}{\textbf{PS=0.10}}}  & 0.1  & 98.48  & 0.63667  & 10.421  & 0.0257  & 97.31  & 0.87222  & 10.516  & \textbf{0.0484} \tabularnewline
% & 0.2 & 93.98 & 0.71667 & 10.043 & 0.098407 & 93.86 & 0.58778 & 9.915 & \textbf{0.1084} \\
 & 0.3  & 82.78  & 0.85778  & 9.05  & \textbf{0.2562}  & 92.08  & 0.47667  & 9.637  & 0.1385 \tabularnewline
% & 0.4 & 62.28 & 1.1078 & 7.225 & \textbf{0.4915} & 83.82 & 0.23556 & 8.594 & 0.2597 \\
 & 0.5  & 38.18  & 1.1511  & 4.854  & \textbf{0.7073}  & 73.69  & 0.40889  & 7.737  & 0.3675 \tabularnewline
% & 0.6 & 18.97 & 1.2544 & 3.026 & \textbf{0.8395} & 35.29 & 0.85222 & 4.296 & 0.7196 \\
 & 0.7  & 7.44  & 1.0178  & 1.66  & \textbf{0.9176}  & 17.54  & 1.4233  & 3.035  & 0.8303 \tabularnewline
% & 0.8 & 2.63 & 0.82444 & 1.005 & \textbf{0.9520} & 3.23 & 1.3344 & 1.524 & 0.9290 \\
% & 0.9 & 0.65 & 0.71222 & 0.706 & \textbf{0.9672} & 1.79 & 1.3322 & 1.378 & 0.9369 \\
 & 1  & 0.11  & 0.61556  & 0.565  & \textbf{0.9740}  & 0.38  & 0.94444  & 0.888  & 0.9595 \tabularnewline
 & 1.25  & 0  & 0.60111  & 0.541  & \textbf{0.9752}  & 0.02  & 0.67  & 0.605  & 0.9722 \tabularnewline
 & 1.5  & 0  & 0.46889  & 0.422  & \textbf{0.9806}  & 0  & 0.62444  & 0.562  & 0.9742 \tabularnewline
\hline 
%& 2 & 0 & 0.63444 & 0.571 & 0.97381 & 0 & 0.3144 & 0.283 & \textbf{0.9868} \\
%& 3 & 0 & 0.58556 & 0.527 & 0.97577 & 0 & 0.1467 & 0.132 & \textbf{0.994} \\ 
\multirow{7}{*}{\rotatebox[origin=c]{90}{\textbf{PS=0.15}}}  & 0.1  & 98.347  & 0.6259  & 15.284  & 0.0293  & 97.513  & 0.84824  & 15.348  & \textbf{0.0460} \tabularnewline
% & 0.2 & 93.68 & 0.77529 & 14.711 & 0.1066 & 95.86 & 0.58706 & 14.878 & \textbf{0.0763} \\
 & 0.3  & 81.48  & 1.0471  & 13.112  & \textbf{0.2832}  & 91.853  & 0.31059  & 14.042  & 0.1431 \tabularnewline
% & 0.4 & 59.8 & 1.3329 & 10.103 & \textbf{0.5307} & 83.513 & 0.26353 & 12.751 & 0.2566 \\
 & 0.5  & 35.48  & 1.4965  & 6.594  & \textbf{0.7390}  & 56.067  & 0.77882  & 9.072  & 0.5457 \tabularnewline
% & 0.6 & 16.747 & 1.4588 & 3.752 & \textbf{0.8672} & 32.767 & 1.5247 & 6.211 & 0.7338 \\
 & 0.7  & 6.5533  & 1.2553  & 2.05  & \textbf{0.9319}  & 10.033  & 1.9471  & 3.16  & 0.8924 \tabularnewline
% & 0.8 & 2.26 & 0.92118 & 1.122 & \textbf{0.9638} & 5.1867 & 2.1412 & 2.598 & 0.9173 \\
% & 0.9 & 0.6 & 0.76941 & 0.744 & \textbf{0.9763} & 0.86 & 1.9 & 1.744 & 0.9464 \\
 & 1  & 0.11333  & 0.64706  & 0.567  & \textbf{0.9821}  & 0.19333  & 1.6247  & 1.41  & 0.9567 \tabularnewline
 & 1.25  & 0  & 0.60824  & 0.517  & \textbf{0.9837}  & 0  & 0.96471  & 0.82  & 0.9744 \tabularnewline
 & 1.5  & 0  & 0.60941  & 0.518  & \textbf{0.9836}  & 0  & 0.68  & 0.578  & 0.9818 \tabularnewline
\hline 
%& 2 & 0 & 0.68353 & 0.581 & 0.98173 & 0 & 0.41412 & 0.352 & \textbf{0.98886} \\
%& 3 & 0 & 0.62706 & 0.533 & 0.98321 & 0 & 0.21882 & 0.186 & \textbf{0.99408} \\ 
\multirow{7}{*}{\rotatebox[origin=c]{90}{\textbf{PS=0.25}}}  & 0.1  & 98.564  & 0.64533  & 25.125  & 0.0264  & 98.312  & 0.86  & 25.223  & \textbf{0.0323} \tabularnewline
% & 0.2 & 94.444 & 0.86933 & 24.263 & \textbf{0.0982} & 96.78 & 0.46267 & 24.542 & 0.0611 \\
 & 0.3  & 83.508  & 1.184  & 21.765  & \textbf{0.2641}  & 93.74  & 0.272  & 23.639  & 0.1128 \tabularnewline
% & 0.4 & 64.424 & 1.576 & 17.288 & \textbf{0.4961} & 85.728 & 0.44133 & 21.763 & 0.2284 \\
 & 0.5  & 40.056  & 1.86  & 11.409  & \textbf{0.7173}  & 63.888  & 1.124  & 16.815  & 0.4648 \tabularnewline
% & 0.6 & 20.688 & 1.82 & 6.537 & \textbf{0.8557} & 49.908 & 1.4267 & 13.547 & 0.5919 \\
 & 0.7  & 8.856  & 1.604  & 3.417  & \textbf{0.9292}  & 13.84  & 3.1027  & 5.787  & 0.8776 \tabularnewline
% & 0.8 & 3.036 & 1.2827 & 1.721 & \textbf{0.9657} & 2.812 & 3.4387 & 3.282 & 0.9374 \\
% & 0.9 & 0.776 & 0.94667 & 0.904 & \textbf{0.9823} & 1.572 & 3.456 & 2.985 & 0.9437 \\
 & 1  & 0.184  & 0.78  & 0.631  & \textbf{0.9878}  & 0.348  & 2.7707  & 2.165  & 0.9593 \tabularnewline
 & 1.25  & 0.004  & 0.65333  & 0.491  & \textbf{0.9905}  & 0.008  & 1.8747  & 1.408  & 0.9732 \tabularnewline
 & 1.5  & 0  & 0.60133  & 0.451  & \textbf{0.9913}  & 0  & 1.1693  & 0.877  & 0.9831 \tabularnewline
%& 2 & 0 & 0.66667 & 0.5 & \textbf{0.99034} & 0 & 0.78533 & 0.589 & 0.98863 \\
% & 3 & 0 & 0.64133 & 0.481 & \textbf{0.99069} & 0 & 0.348 & 0.261 & 0.99493
\end{tabular}\caption{Diagnosis simulation results for Scenario III \label{tab:Diag_compare_sen4} }
\end{table}

\begin{figure}[!htb]
\centering \subfloat[PS=0.10]{\includegraphics[width=0.45\textwidth]{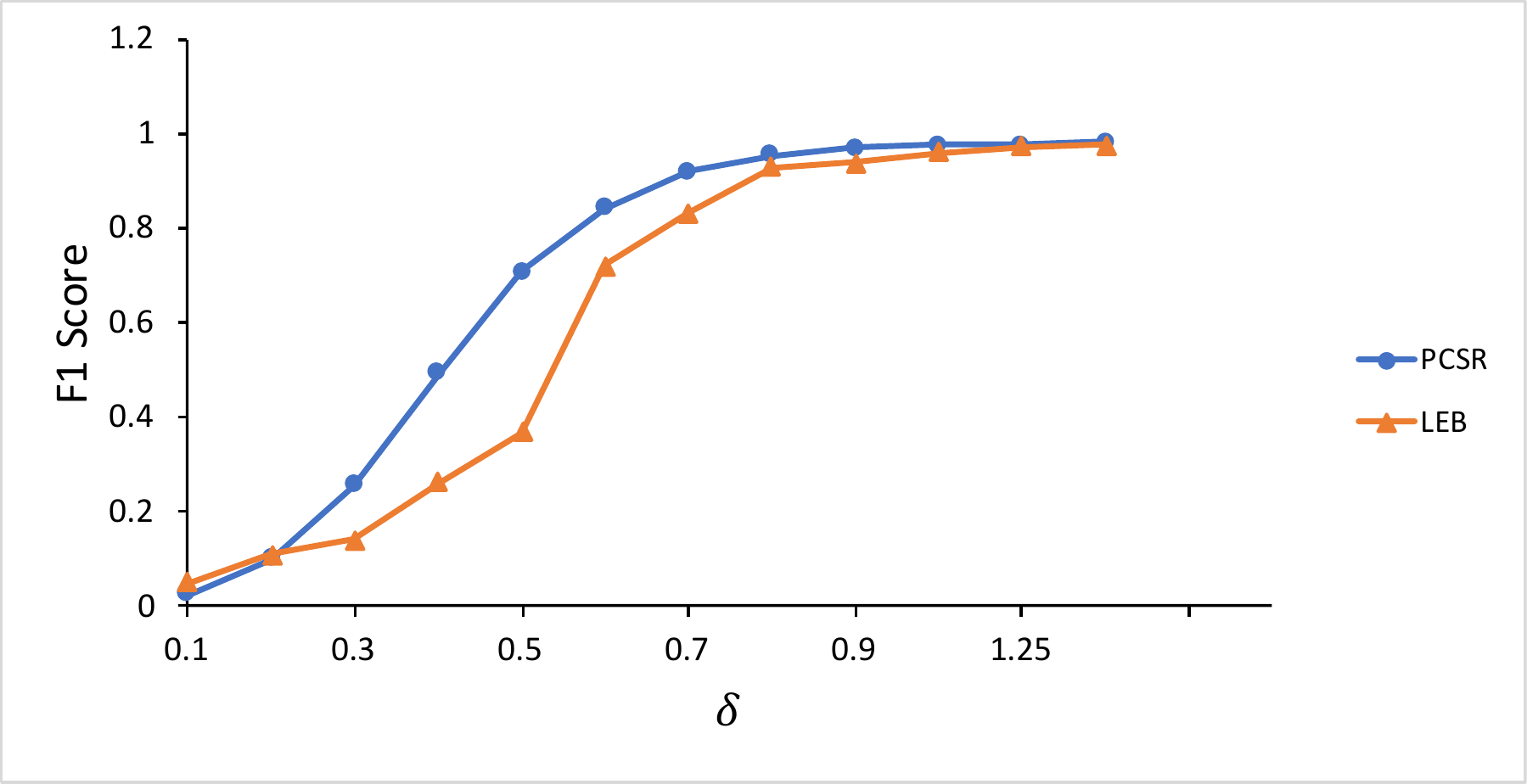}}
\vspace{0.05\linewidth}
\subfloat[PS=0.15]{\includegraphics[width=0.45\textwidth]{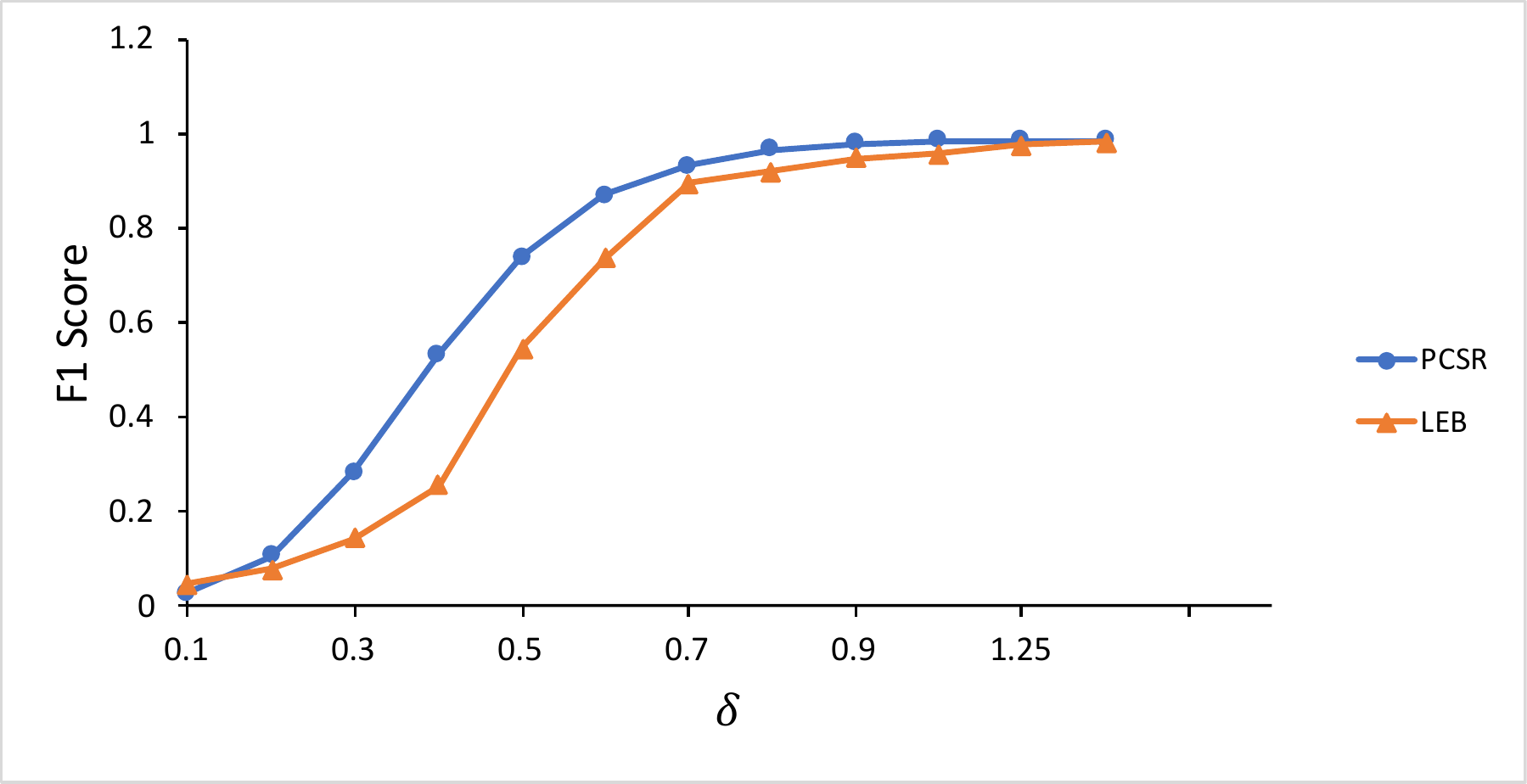}}
\hspace{0.5mm} \subfloat[PS=0.25]{\includegraphics[width=0.45\linewidth]{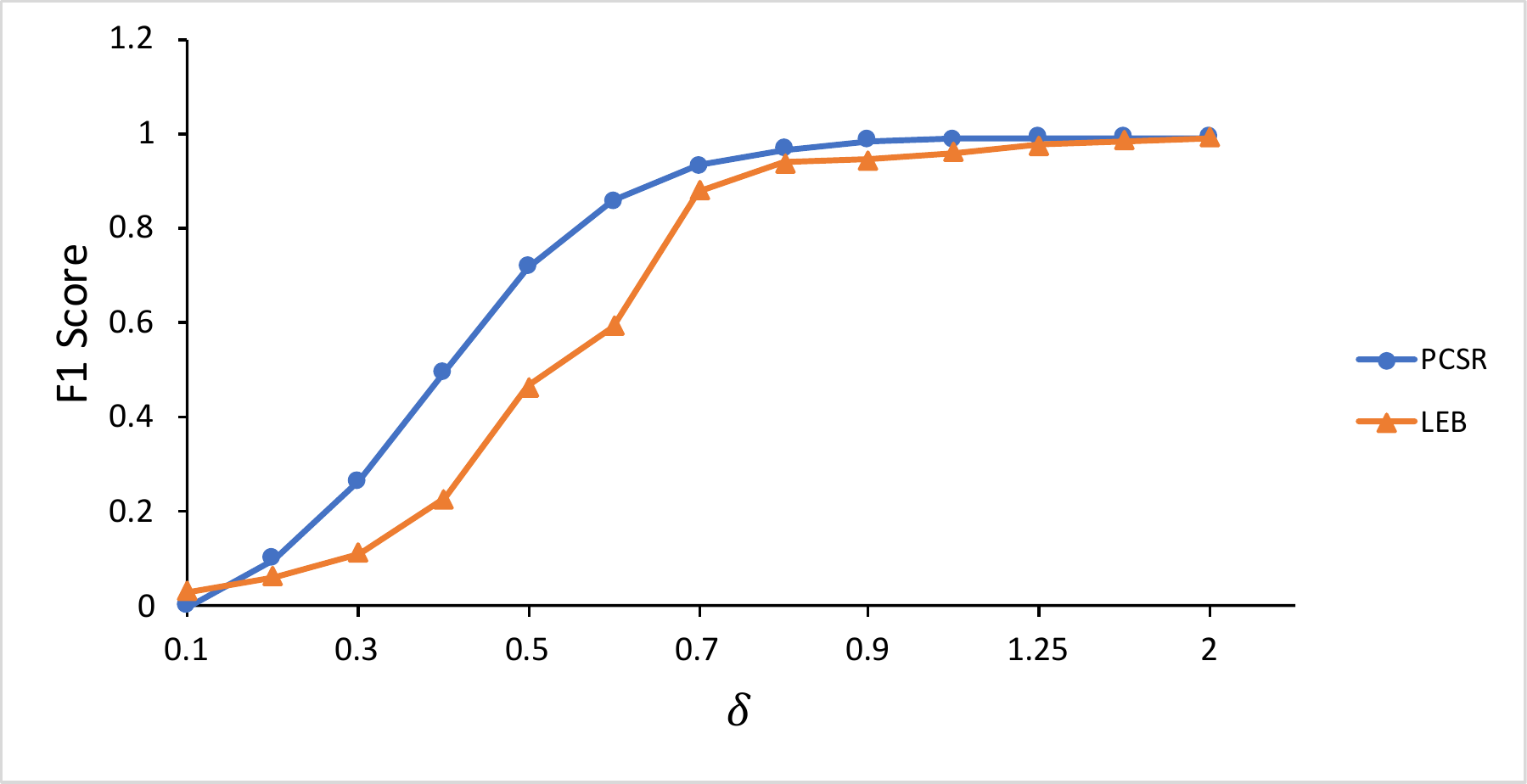}}
\caption{F1 of scenarios III different values of $\delta$ (shift magnitude)
}
\label{fig:F1_sen4_compare} 
\end{figure}

\section{Case Study \label{sec:Case-Study}}

In this section, we apply the proposed monitoring and diagnosis methods
on two case studies, a) defect detection in a steel rolling process,
and b) quality monitoring of wine. Additionally, we compare our results
with the existing methods.

\subsection{Defect detection in Steel Rolling Process}

Early detection of process shifts in a rolling process is necessary
to avoid damage to products and reduce manufacturing costs. Rolling
is a high-speed process that makes its monitoring particularly challenging.
In this study, we show that the PCA-based method can effectively
detect anomalies and damages imprinted on a steel bar after rolling.
The dataset we consider here, includes images of size 128$\times$512
pixels of the surface of rolled bars collected by a high-speed camera
\cite{yan2017anomaly}. Of the 100 images, the first 50 images are
in-control. One example of the image of rolling data for in-control vs
out-of-control process is shown in Figure~\ref{fig:rolling_image_beforeafter}.

We use this data to simulate an image with in-control observations
in the first 126 rows and out-of-control observations in the remaining
72 rows. The generated image is presented in Figure \ref{fig:generate_image}.
Also, we crop the image at the right end to avoid the non-informative dark segment of the image. Hence, our generated picture is of the size
of 198$\times$300. In this study, each row of an image (a vector
of 300$\times$1 ) is treated as an observation, creating a multi-stream data with the size of 300. As the picture shows, for out-of-control
observations, some small black lines, indicating anomalies, emerge
at the left part of the frame. We are interested to see whether our
monitoring approach can detect this change, and whether the diagnosis
approach can determine the changed pixels.

\begin{figure}
\centering{} %
\noindent\begin{minipage}[t]{0.31\paperwidth}%
\begin{center}
\subfloat[Rolling image at time 1]{\begin{centering}
\includegraphics[width=0.28\paperwidth]{Rolling1} 
\par\end{centering}
}
\par\end{center}%
\end{minipage}%
\noindent\begin{minipage}[t]{0.31\paperwidth}%
\begin{center}
\subfloat[Rolling image at time 90]{\centering{}\includegraphics[width=0.28\paperwidth]{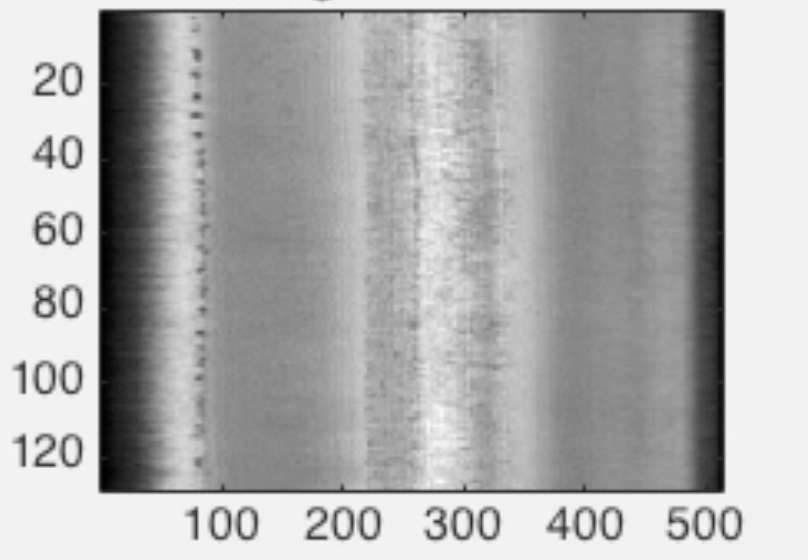}

}
\par\end{center}%
\end{minipage}\caption{Image of rolling data for in control process (a) and out of control
process (b)}
\label{fig:rolling_image_beforeafter} 
\end{figure}

\begin{figure}
\centering{} \includegraphics[width=0.6\paperwidth,height=0.2\textheight]{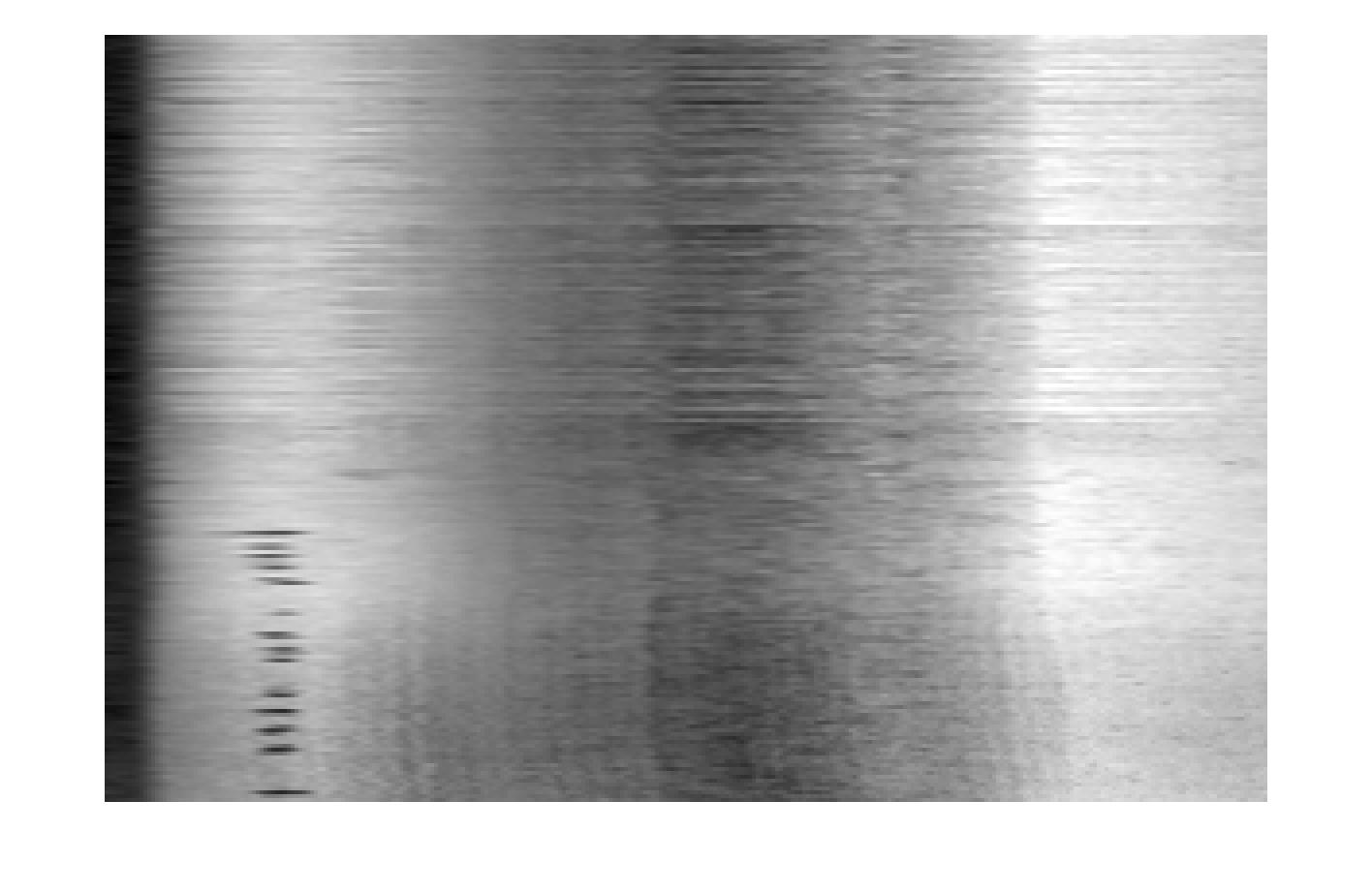}
\caption{Generated Image with first 126 rows (from the top) as in-control and remaining 72
rows as out-of-control \label{fig:generate_image}}
\end{figure}

\begin{comment}
in each image there are streams of 128 strips, and with 100 images,
we will have 12,800 (128$\times$100) data stream in total. Of the
12,800 data sample, we use the first 5,120 in-control data (the first
40 images) to obtain the control limits. The control limits are found
as explained in Sec. \ref{subsec:Monitoring-Methods-Analysis} to
obtain in-control ARL of 200. 
\end{comment}

\textbf{Monitoring}. We apply our proposed APC method for monitoring the process. We use
the first 70 in-control data (the first 70 rows of the image) to obtain
the control limits. The control limits are determined according to
the procedure explained in Sec. \ref{subsec:Monitoring-Methods-Analysis}
to achieve the in-control ARL of 200. The resulting control chart
is shown in Figure~\ref{fig:Monitoring-Rolling-Data}. As can be seen
from the figure, after the change point, our monitoring statistic
instantly inflates and raises an out-of-control alarm by the first
observation after the change. Furthermore, to compare its performance
with existing state-of-the-art methods, we report the run
lengths (the number of observations before the change is detected)
for each method in Table~\ref{tab:Comparison-of-different-1}. In the results, APC
has the smallest run-length (RL). This implies APC is the fastest in detecting
the change in comparison to other benchmarks.

\begin{comment}
The result for our APC monitoring method is shown in Figure~\ref{fig:Monitoring-Rolling-Data}.
Also comparing the RL after the change for APC, LPC, Conventional
PCA, Zou's method and Hao's method is provided in \ref{tab:Comparison-of-different}.
As this result shows LPC is fastest in detecting the change. 
\end{comment}

\begin{figure}
\centering{} 
\includegraphics[width=0.8\paperwidth]{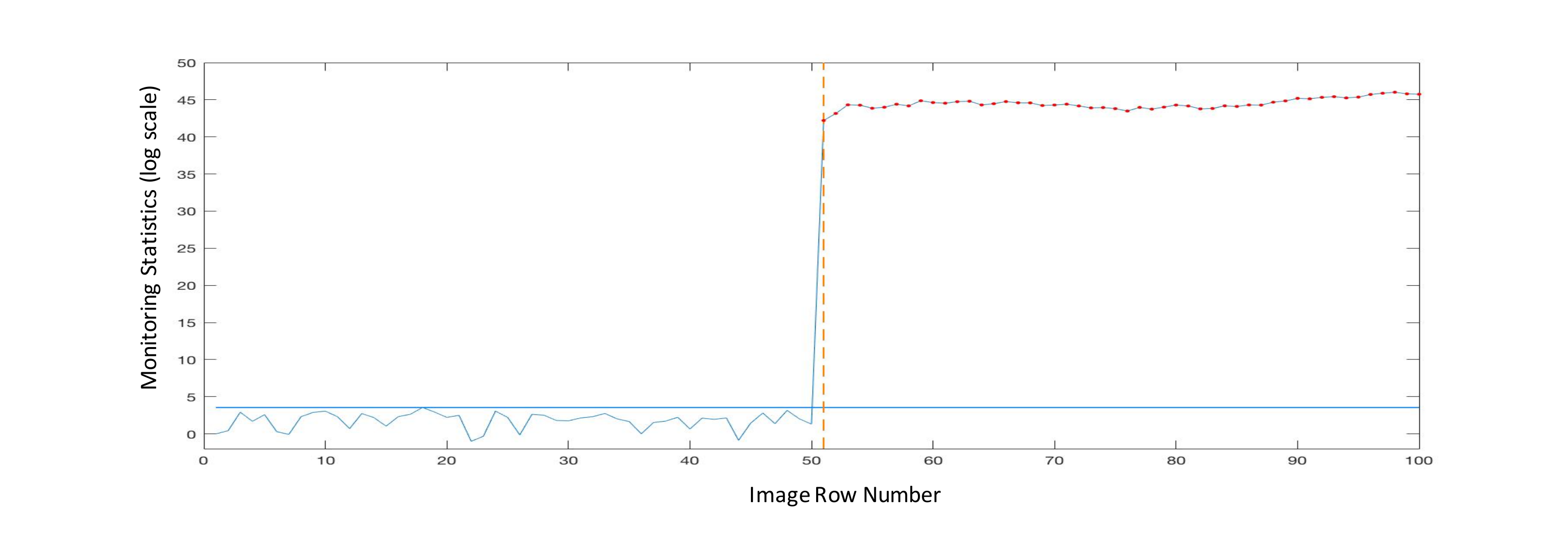}
\caption{Monitoring Rolling Data using APC Method \label{fig:Monitoring-Rolling-Data}}
\end{figure}

\begin{table}[t]
\centering %
\begin{tabular}{l|c}
\textbf{Method}  & \textbf{Detected Change Point} \tabularnewline
\hline 
%LPC              & \textbf{16}                             \\
APC  & \textbf{1} \tabularnewline
Conventional PCA  & 16 \tabularnewline
T\_new (~\cite{zou2015efficient})  & 10 \tabularnewline
TRAS (~\cite{liu2015adaptive})  & 14 \tabularnewline
\hline 
%~\cite{yan2017anomaly}       & 26 \\
 & \tabularnewline
\end{tabular}\caption{Run length Comparison of different methods in detecting the change
point}
\label{tab:Comparison-of-different-1} 
\end{table}

\textbf{Diagnosis}. To check the performance of our PCSR method,
we performed diagnosis using our method vs LEB method on the out of
control data. The phase-1 data is used as the ground truth (sample size 70), and
25 out-of-control observations are used to detect the changed pixels
in the generated image. The area selected as out-of-control for each
method as well as the in-control and out-of-control images are shown
in Figure \ref{fig:rolling_diagnosis}. The identified pixels are shown in black and the remaining
unchanged pixels are shown in white in Figure \ref{fig:rolling_diagnosis}(c)
and (d), respectively. 

As the results show, PCSR method clearly detects
the changed pixels in the image with no false detection. Note that
although LEB can identify the changed pixels, it generates a few false
detection areas. 

\begin{figure}
\centering{}
\noindent\begin{minipage}[t]{0.61\paperwidth}%
\begin{center}
\subfloat[In-control image]{\begin{centering}
\includegraphics[width=0.61\paperwidth]{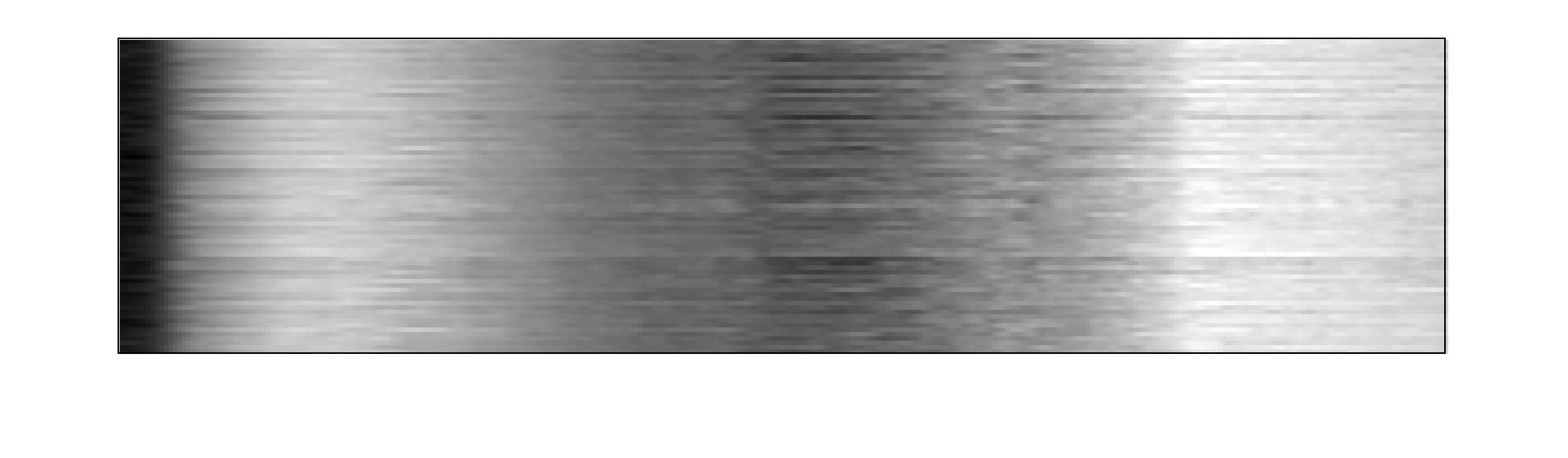}
\par\end{centering}
}
\par\end{center}%
\end{minipage}%
\\
\noindent\begin{minipage}[t]{0.61\paperwidth}%
\begin{center}
\subfloat[Out-of-control image]{\begin{centering}
\includegraphics[width=0.61\paperwidth]{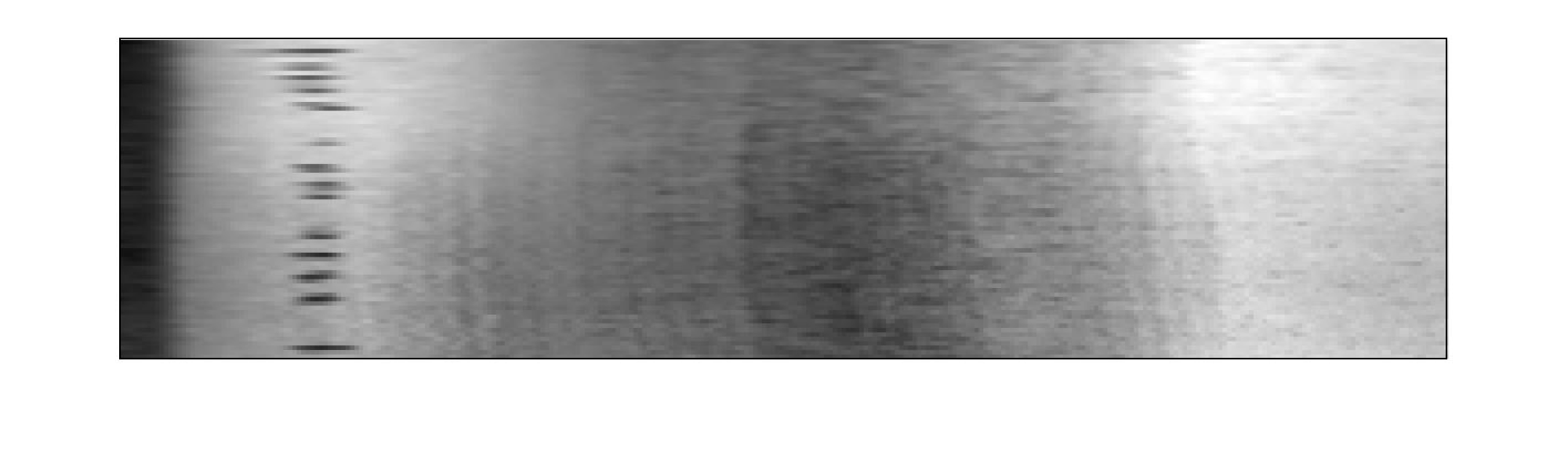}
\par\end{centering}
}
\par\end{center}%
\end{minipage}%
\\
\noindent\begin{minipage}[t]{0.61\paperwidth}%
\begin{center}
\subfloat[Diagnosis using PCSR]{\begin{centering}
\includegraphics[width=0.61\paperwidth]{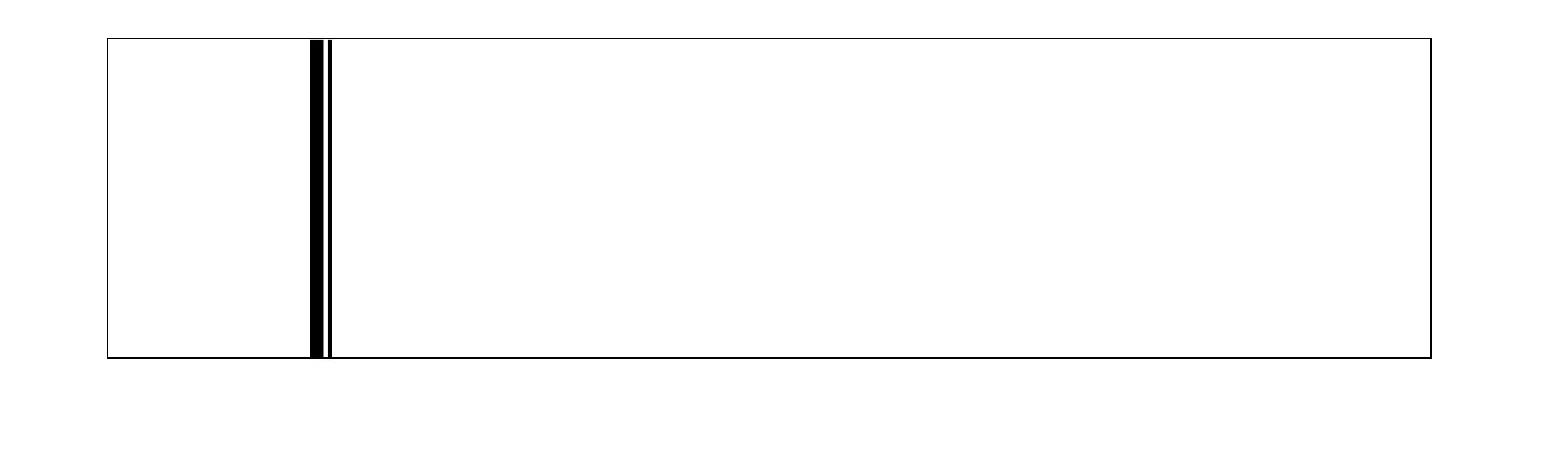}
\par\end{centering}
}
\par\end{center}%
\end{minipage}%
\\
\noindent\begin{minipage}[t]{0.61\paperwidth}%
\begin{center}
\subfloat[Diagnosis using LEB]
{\begin{centering}
\includegraphics[width=0.61\paperwidth]{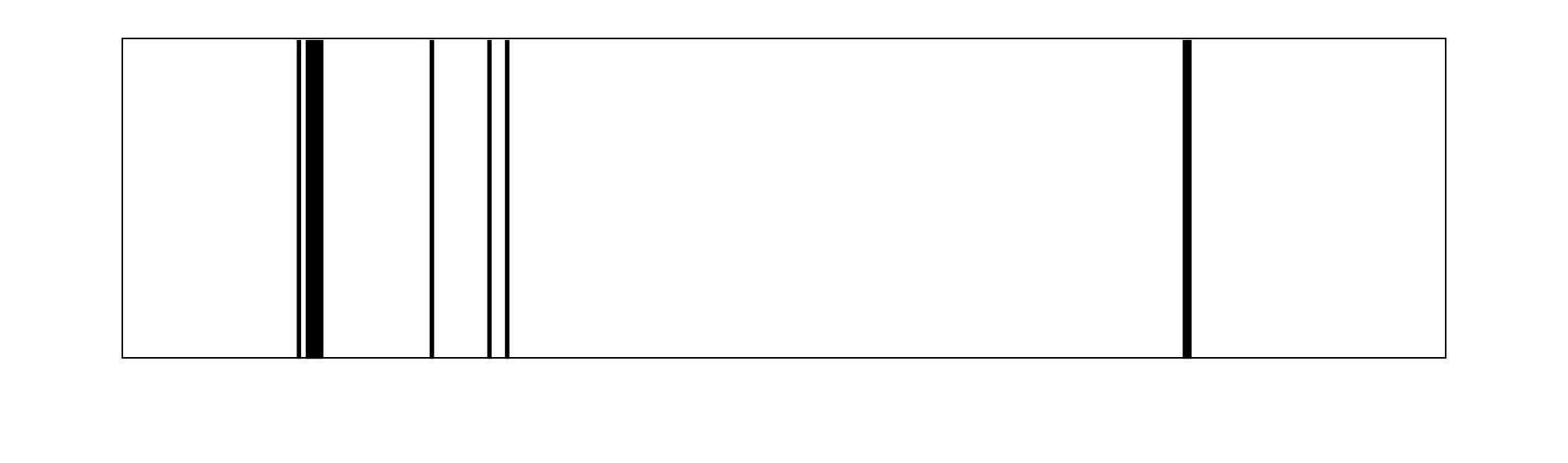}
\par\end{centering}
}
\par\end{center}%
\end{minipage}
\caption{Diagnosis using PCSR and LEB method}
\label{fig:rolling_diagnosis}
\end{figure}

\subsection{Wine Quality Monitoring}

In this section, we demonstrate the efficacy of our proposed methodology
by applying it to a real dataset from a white wine production process.
The data is taken from the UCI data repository \footnote{http://archive.ics.uci.edu/ml/datasets/Wine+Quality}.
The data has 4898 observations obtained between May 2004 to February
2007 for the purpose of improving the quality of Portuguese Vinho
Verde wine. The collected data has eleven variables named as fixed
acidity, volatile acidity, citric acid, residual sugar, chlorides,
free sulfur dioxide, total sulfur dioxide, density, PH, sulphates
and alcohol. An additional (manually annotated) quality
variable is available that will be used as ground truth for the
wine quality. This variable ranges between 0 (very bad) and 10 (very
excellent), and is provided based on sensory analysis \cite{cortez2009modeling}.

Our objective is to monitor the wine quality using
the variables and diagnose the shifted variables, if there is a shift.
We perform the APC study on this dataset. Similar to \cite{zou2015efficient}'s
study on this data,
we focus on a subset of the data in which the quality variable is
either 6 or 7. The observations with the quality of 7 are considered as
acceptable while the rest are unacceptable, hence out-of-control. The quality variable is the ground truth that is used to gauge the performance of our monitoring\textemdash the monitoring should raise an out-of-control alarm as soon as the quality variable is going down from 7 to 6. When the alarm is raised, our diagnosis approach should be able to pinpoint the actual shifted process variables.

\textbf{Monitoring}. Overall there are 880 observations with the quality equal to 7 of
which 830 observations are used for phase I monitoring. Also, we set
the control limits ($R_{0}$ in APC method) such that ARL for in-control
observation is 1000. To do the comparison, we implement our method
along with the existing methods shown in Sec. \ref{subsec:Monitoring-Methods-Analysis}.
All parameters in the methods are set to achieve in-control ARL of
1000 so that the methods are comparable.

For phase II monitoring, we use the remaining 40 points with the quality
of 7 followed by observations with the quality of 6. The goal is to investigate
how fast and accurately our monitoring algorithm detects the change
point in comparison to the existing methods.

The results are shown in Table~\ref{tab:Comparison-of-different} and Figure~\ref{fig:Monitoring-Wine-Quality}. As shown in Table ~\ref{tab:Comparison-of-different},
the APC monitoring method detects the change after $11$ observations.
On the other hand, other methods took more than twice as many observations to detect the change.

\begin{figure}[h]
\centering{} \includegraphics[width=1\linewidth]{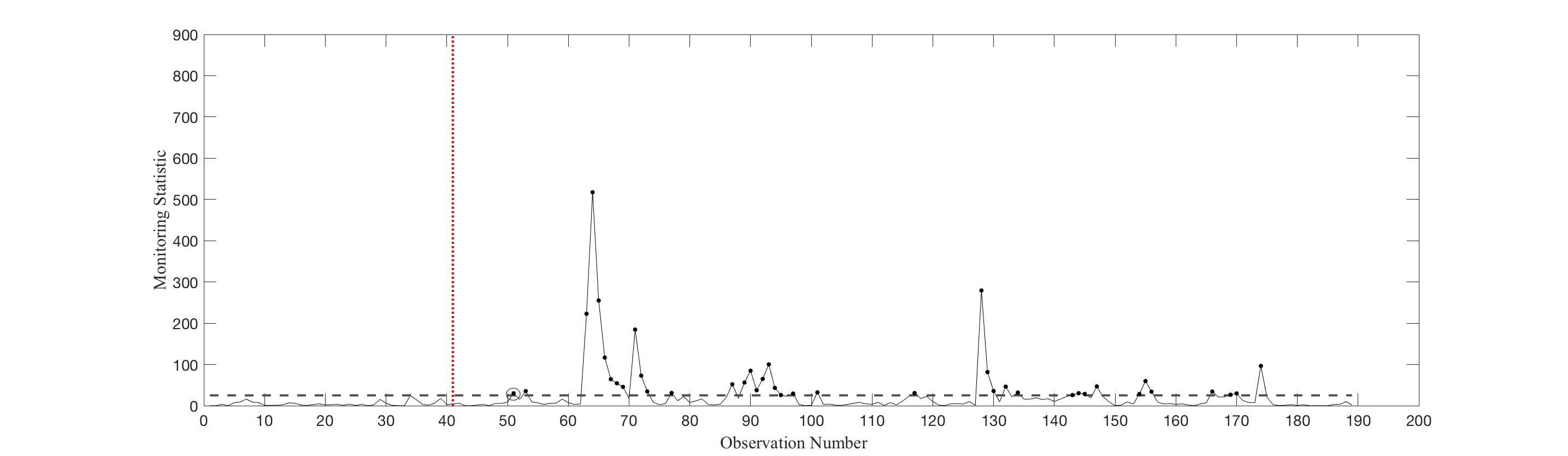}
\caption{Monitoring Wine Quality Data using APC Method \label{fig:Monitoring-Wine-Quality}}
\end{figure}

\begin{table}[t]
\centering %
\begin{tabular}{l|c}
\textbf{Method}  & \textbf{$N_{A}$}$=$Number of observations after Chang point until
alarm \tabularnewline
\hline 
%LPC             & \textbf{10}          \\
APC  & \textbf{11} \tabularnewline
Conventional PCA  & 23 \tabularnewline
T\_new \cite{zou2015efficient}  & 24 \tabularnewline
TRAS \cite{liu2015adaptive}  & 28 \tabularnewline
\hline 
\end{tabular}\caption{Comparison of different methods in detecting the change \label{tab:Comparison-of-different} }
\end{table}

\textbf{Diagnosis}. %\subsubsection*{Diagnosis}After running our
Among the eleven variables, four were determined
as shifted variables by PCSR, viz. residual sugar, chlorides, density, and alcohol. 
The LEB method selected
chlorides, density and alcohol as the shifted variables. 

\section{Conclusion \label{sec:Conclusion-and-future}}

In this paper, we proposed an SPC framework for high-dimensional data
streams that seamlessly integrates monitoring and diagnostics. We
proposed a new PCA-based monitoring approaches, viz. Adaptive PC Selection
(APC) monitoring. We first negated the common belief that the high-PCs
(principal components with highest variances) should be used for
monitoring, and then, showed that monitoring adaptively selected
PCs will be more effective. %The low-PCs are shown (using theoretical examples and experimental results) to have better performance when the shift occurs in a random set of process variables, while 
Using simulations, we showed that adaptively selected PCs outperforms
other benchmark methods for different types of covariance matrix structures
and types of shifts. Moreover, in all the stated scenarios, the conventional
approach of monitoring high-PC was shown to have poorer performance.

%In the diagnosis module, we first discuss the challenge in performing diagnosis after a PCA-based monitoring because a shift signaling PC will have several process variables, thus, difficult to isolate. 
In the diagnosis module, we first discussed the challenge in finding
the shifted variables after a PCA-based monitoring procedure. The
challenge lies in isolating the process variable from the signaling PC. %We, then, show that we can use the principles of Compressed Sensing (CS) to perform diagnostics in such a case. We used the CS principle to formulate an Adaptive Lasso estimation problem (Lass-PC) that takes in the process eigenvector matrix and the principal scores after a shift is detected to yield the actual process variables that caused the shift. We tested Lass-PC's performance using experimental studies where it significantly outperformed the existing state-of-the-art method.
To address this, we used the CS principle to formulate an adaptive Lasso estimation
to detect the shifted variables. This formulation takes the eigenvectors
and principal components (after a shift) as inputs and yields the
process variables that caused the shift. Our experimental validations
showed that the proposed PCSR performs significanlty better than the current state-of-the-art.

%After the experimental analysis of our monitoring and diagnostics methods, we applied them on two real-world data to show their application.
Furthermore, we showed the practical applicability and validity of
our methods via real-world case studies. The first case study was
on defect detection in a steel rolling process, in which we found
that the proposed APC detects the shift faster than all the other
methods. Moreover, the PCSR diagnosis approach detects the change
pixels better than the existing method with fewer false positives. In another case study, we
monitored wine quality and diagnosed the shift. %-> "In another case study, we monitored wine quality from a real data stream, found a shift, and diagnosed it to find the shift causing process variables."
Our monitoring approach was again faster and our diagnosis approach
could find an additional shifted process variable,\emph{ sugar}, that
was undetected by the existing diagnostics approach. 

In this paper, we
have focused on monitoring and diagnosing the mean shifts. While the
developed APC can potentially be used to detect shifts in covariances,
further research is required to extend the PCSR diagnostics approach
to the covariance matrix monitoring.

\bibliography{bibliography} 

%\begin{appendices}
\appendix
%Begin individual appendices, separated as chapters
\section{First and second moments of the thresholded statistic}
\label{Ap:prop_dtilde_prove}
The first moment of $\tilde{d}$ is calculated as,
\begin{equation}
\label{eq:E_d_tilde}
\begin{aligned}
E(\tilde{d})&=E((d_{tj}-\nu)_{+})=E((d_{tj}-\nu)|d_{tj}>\nu)P(d_{tj}>\nu)\\
&=E((d_{tj}|d_{tj}>\nu)p(d_{tj}>\nu)-\nu P(d_{tj}>\nu)\\
&=\int_{\nu}^{\infty}d P(d)-\nu P(\chi_{1}^2>\nu)\\
&=\int_{\nu}^{\infty}\{x \frac{1}{\sqrt{2}\Gamma(0.5)}x^{-\frac{1}{2}}e^{-\frac{x}{2}}\} dx-\nu P(\chi_{1}^2>\nu)\\
&=\frac{1}{\Gamma(0.5)}\{\Gamma(0.5,\frac{\nu}{2})+e^{-\frac{\nu}{2}}\sqrt{2\nu}\}-\nu P(\chi_{1}^{2}>\nu)
\end{aligned}
\end{equation}

To calculate the second moment of $\tilde{d}$,
\begin{equation}
\begin{aligned}
E(\tilde{d}^{2})&=E((d_{tj}-\nu)_{+}^{2})=E((d_{tj}-\nu)_{+}^{2})|d_{tj}>\nu) P(d_{tj}>\nu)\\
=&E(d_{tj}^{2}|d_{tj}>\nu) P(d_{tj}>\nu)-2\nu E(d_{tj}|d_{tj}>\nu) P(d_{tj}>\nu)+\nu^2 P(d_{tj}>\nu) \overset{\mathrm{Eq. \ref{eq:E_d_tilde}}}{\rightarrow}\\
=&\int_{\nu}^{\infty}\{x^2 \frac{1}{\sqrt{2}\Gamma(0.5)}x^{-\frac{1}{2}}e^{-\frac{x}{2}}\} dx-2\nu \{E(\tilde{d})+\nu P(\chi_{1}^{2}>\nu)\}+\nu^{2}P(\chi_{1}^{2}>\nu)\\
=&\int_{\nu}^{\infty}\{x^2 \frac{1}{\sqrt{2}\Gamma(0.5)}x^{-\frac{1}{2}}e^{-\frac{x}{2}}\} dx-2\nu E(\tilde{d})-\nu^{2}P(\chi_{1}^{2}>\nu)\\
=&\frac{1}{\Gamma(0.5)}[3\Gamma(0.5,\frac{\nu}{2})+e^{-\frac{\nu}{2}}\sqrt{2\nu}(3+\nu)]-2\nu E(\tilde{d})-\nu^{2}P(\chi_{1}^{2}>\nu)
\end{aligned}
\end{equation}

\section{Consistency of the Diagnosis Method}
\label{Ap:consistency}
To prove the consistency of our diagnosis model, we use the derivations in \cite{zou2006adaptive}. To show the consistency in adaptive lasso \citeauthor{zou2006adaptive} used the following conditions,  
\begin{condition}\label{cond_1}
noise have independent identical distribution with mean 0 and variance $\sigma^2$
\end{condition}
\begin{condition}\label{cond_2}
for observation matrix \textbf{X},and number of observations n, $\frac{1}{n} \textbf{X}^T \textbf{X}\rightarrow \textbf{C}$. Where $\textbf{C}$ is a positive definite matrix.
\end{condition}

Condition~\ref{cond_1} is valid for Eq.~\ref{eq:lasso_transform}. As we showed, after the transformations, the model noise has iid distribution with variance equal to 1.
To show the validity of condition~\ref{cond_2} we need to show that this consition holds for $\textbf{A}^*$ instead of $\textbf{X}$. We use lemma~\ref{lem:positive} and its proof to show it.

\begin{lemma}
For $\textbf{A}^*$ given in Eq.~\ref{eq:weightedL1} , $\textbf{C}=\frac{1}{p} {\textbf{A}^*}^T \textbf{A}^*$ is a positive definite matrix
\label{lem:positive}
\end{lemma}

\begin{proof} The proof is a follows: 
\begin{align*}
A^*&=\Lambda^{\frac{-1}{2}}A \nonumber \\
\textbf{C}=\frac{1}{p} {\textbf{A}^*}^T \textbf{A}^*= \frac{1}{p} A^T & \Lambda^{\frac{-1}{2}}  \Lambda^{\frac{-1}{2}} A \Rightarrow
 \textbf{C}= \textbf{A}^T \Lambda^{-1} \textbf{A}  \nonumber 
\end{align*}

To show that C is positive definite matrix it suffice to show that $x^T \textbf{C} x>0, \forall x\neq 0 $
\begin{align*}
x^T \textbf{C} x&=x^T \textbf{A}^T \Lambda^{-1} \textbf{A} x \\
&=z^T \Lambda^{-1}  z\\
&=\sum_{i=1}^p \frac{1}{\lambda_i} z_i^2
\end{align*}
Where $\lambda_i$ is the pc score i. Since pc scores are positive, hence:
\begin{equation*}
x^T \textbf{C} x>0, \forall x\neq 0 \qedhere
\end{equation*}
\end{proof}

Since our model holds the above conditions, we can now show the consistency of our model, as follows:

\begin{theorem}
Suppose that $\lambda$ in Eq.~\ref{eq:weightedL1} varies with $p$. If $\dfrac{\lambda_p}{\sqrt{p}}\rightarrow 0$, and $\lambda_p \rightarrow \inf$, then the adaptive lasso estimate must satisfy the following:
\begin{itemize}
\item Consistency in variable selection: $lim_p P(\textit{S}^*=\textit{S})=1$ 
\item Asymptotic Normality: $\sqrt{p}(\hat{\mu}^*_\textit{S}-\hat{\mu}_\textit{S}\rightarrow_d N(\textbf{0},\sigma^\prime)$

Where $\sigma^\prime=\textbf{C}_{11}^{-1}$
\end{itemize}
\end{theorem}
%\todo[inline]{talk about $\textit{S}$ and $\textbf{C}_{11}^{-1}$}

\begin{proof}
for proof of Theorem, please refer to Theorem 2 in  \cite{zou2006adaptive} \qedhere
\end{proof}

\end{document}